\documentclass[10pt,twocolumn,twoside]{IEEEtran}
\usepackage{url}                
\usepackage{times}
\usepackage{epsfig}
\usepackage{graphicx}
\usepackage{amsmath}
\usepackage{dsfont}
\usepackage{amssymb}
\usepackage{subfigure}          
\usepackage{multirow}           
\usepackage{bm}
\usepackage{color}
\usepackage{soul}
\usepackage{stfloats}
\usepackage{lineno}
\usepackage{slashbox}

\usepackage[linesnumbered,ruled]{algorithm2e}  
\usepackage{algorithmic} 

\newtheorem{proposition}{Proposition}

\def\ie{\emph{i.e.}}
\def\eg{\emph{e.g.}}
\def\etal{\emph{et al.}}

\definecolor{mygray}{gray}{.9}
\def\txtred#1{\textcolor{red}{\textbf{#1}}}
\def\txtblu#1{\textcolor{blue}{\textit{#1}}}

\ifCLASSINFOpdf
\else
\fi
\begin{document}
%
\title{A Hybrid Data Association Framework for Robust Online Multi-Object Tracking}
%
%
%

\author{Min Yang, \IEEEauthorblockN {Yuwei Wu\IEEEauthorrefmark{1}}, and Yunde Jia \IEEEmembership{ Member, IEEE}, 
\thanks{This work was supported in part by the Natural Science Foundation of China (NSFC) under Grants No.61375044 and No. 61472038. (Corresponding author:Yuwei Wu.)}
\thanks{The authors are with Beijing Laboratory of Intelligent Information Technology, School of Computer Science, Beijing Institute of Technology (BIT), Beijing, 100081, P.R. China. Email: \{yangminbit,wuyuwei,jiayunde\}@bit.edu.cn.} }

\maketitle

\begin{abstract}
Global optimization algorithms have shown impressive performance in data-association based multi-object tracking, but handling online data remains a difficult hurdle to overcome. In this paper, we present a hybrid data association framework with a min-cost multi-commodity network flow for robust online multi-object tracking. We build local target-specific models interleaved with global optimization of the optimal data association over multiple video frames. More specifically, in the min-cost multi-commodity network flow, the target-specific similarities are online learned to enforce the local consistency for reducing the complexity of the global data association. Meanwhile, the global data association taking multiple video frames into account alleviates irrecoverable errors caused by the local data association between adjacent frames. To ensure the efficiency of online tracking, we give an efficient near-optimal solution to the proposed min-cost multi-commodity flow problem, and provide the empirical proof of its sub-optimality. The comprehensive experiments on real data demonstrate the superior tracking performance of our approach in various challenging situations.
\end{abstract}



\begin{IEEEkeywords}
Multi-object tracking, data association, optimization, multi-commodity flow.
\end{IEEEkeywords}

%

\section{Introduction}
\label{sec:introduction}
Online multi-object tracking is to estimate the spatio-temporal trajectories of multiple objects in an online video stream (\ie, the video  is provided frame-by-frame), which is a fundamental problem
for numerous real-time applications, such as video surveillance, autonomous driving, and robot navigation. Assume that an object detector is available to detect potential locations of multiple objects in each frame, the tracking problem is consequently reduced to a data association procedure which links these individual detections to form consistent trajectories.

Data association is a challenging problem in many situations, especially in complex scenes, due to the presence of occlusions, inaccurate detections, and interactions among similar-looking objects.
Standard approaches for data association are to recursively link detections frame by frame \cite{breitenstein2011online,shu2012part,possegger2014Occlusion,xiang2015learning,solera2015learning,yang2016temporal,hong2016online,Milan2016arxiv}, resulting in a bi-partite matching between the existing trajectories and the newly obtained detections, as shown in Fig.~\ref{motivation}(a). These approaches are temporally local and computationally efficient, making them suitable for the online setting. However, using only the local information for data association might lead to irrecoverable errors when an object is undetected or is confused with clutters.
To overcome this shortcoming, the global data association over entire video frames (or a batch of frames) has been devoted to inferring optimal trajectories \cite{zhang2008global,berclaz2011multiple,luo2015automatic,chari2015pairwise,dehghan2015gmmcp,tang2015subgraph,wang2016tracklet,tang2016multi,maksai2016globally}, as shown in Fig.~\ref{motivation}(b). Such a data association problem can be solved in an optimization framework with carefully designed cost functions. Unfortunately, global association methods can not be directly applied to online video streams. Overlapping temporal window is a common choice to handle online data \cite{berclaz2011multiple,chari2015pairwise,dehghan2015gmmcp}, but the connection between consecutive batches remains an open problem.

\begin{figure*}[t]
    \centering
    \includegraphics[width=0.98\textwidth]{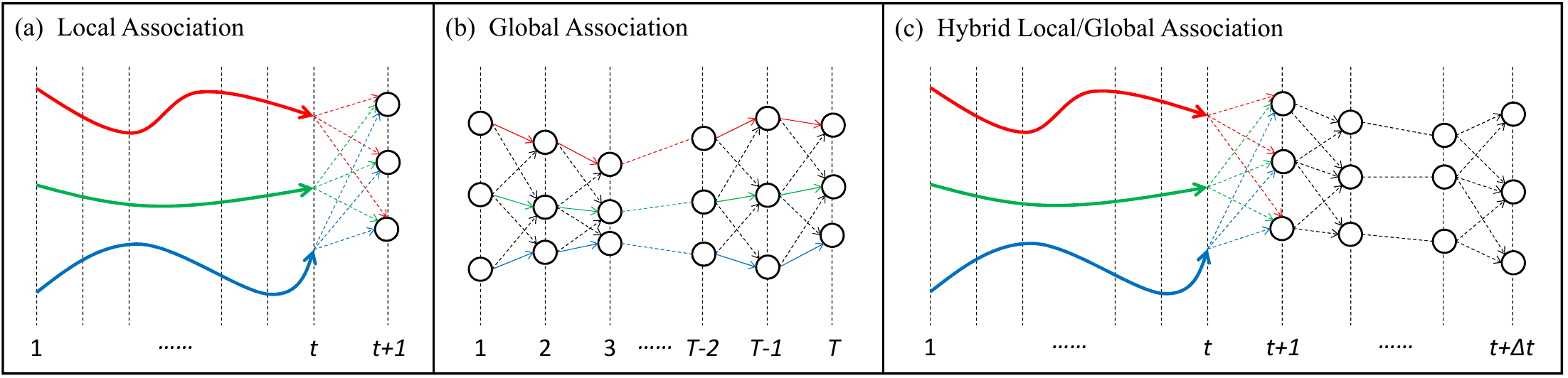}
    \vspace{-0pt}
    \caption{Illustration of our hybrid data association approach. (a) Local association is performed between two consecutive frames $t$ and $(t+1)$, and a bi-partite matching between the existing trajectories (marked as color arrows) at the current frame $t$ and the detections (marked as circles) from the next frame $(t+1)$ is usually solved. (b) Global association is performed over a batch of frames (length $T$ in this example), and a optimization problem is usually solved to infer optimal trajectories based on pairwise affinities between detections. (c) The hybrid association finds globally optimal associations for the existing trajectories within a temporally local window (length $\Delta t$ in this example), and the target-specific information from the existing trajectories provides helpful local constraints to guide the global optimization. }
    \vspace{-0pt}
    \label{motivation}
\end{figure*}

In this paper, we propose a hybrid data association framework for online multi-object tracking, which characterizes the superiorities of both local and global data association methods. The core of our approach lies on the association between the existing trajectories and the detections from multiple video frames within a temporal window, as shown in Fig.~\ref{motivation}(c). We exploit a mini-cost multi-commodity flow which is with respect to a cost-flow network constructed by the detections from multiple frames.
 The proposed mini-cost multi-commodity network is able to formulate a hybrid data association strategy to handle online data with an efficient near-optimal solution.

In our framework, concretely, all possible associations among the detections are represented by edges in the network, where the corresponding edge costs account for the association likelihoods.
Each existing trajectory is then supposed to be a specific commodity, and its optimal associations can be found by sending specific commodity flows through the network with a minimum cost.
To this end, the following three challenges need to be studied: (i) identifying newly appeared objects automatically; (ii) computing edge cost for different commodities; (iii) solving the min-cost multi-commodity flow problem efficiently.
By addressing these challenges, we bring the following three contributions:
\begin{itemize}\setlength{\itemsep}{-0pt}

\item We introduce a dummy commodity into our network to automatically identify a new object. The dummy commodity corresponds to a target-independent model, and its commodity flows indicate the permissible tracks of objects newly appeared in a temporal window.

\item We present an online discriminative appearance modeling approach to build target-specific models for different existing trajectories. The edge costs of multiple commodities in the network are estimated by exploiting the target-specific information to discriminate a specific target from both other targets and the background.

\item We propose a near-optimal solution algorithm to the min-cost multi-commodity flow problem, and provide the empirical proof of its sub-optimality. By using the reformulation and column generation strategy, our solution is extremely efficient and performs superiorly in multi-object tracking.
\end{itemize}

The proposed hybrid strategy offers several advantages over existing methods. First, it makes the global optimization of trajectories applicable to online data. The local association between consecutive frames is extended to account for more hypotheses from multiple frames. Irrecoverable errors caused by noisy detections or frequent occlusions can be alleviated to improve online tracking performance. Second, the target-specific information from the existing trajectories is explicitly modeled to guide the global optimization over the current batch of frames. In practice, it enforces local constraints to reduce the complexity of the optimization problem as the associated detections are restricted to be consistent with the target-specific models.
We believe that the techniques described in this paper are of wide interests due to their efficiency and performance. Both qualitative explanation and experimental confirmation are provided to support this claim.

The rest of the paper is organized as follows. Section \ref{relatdwork} reviews the related work. In Section \ref{MCMCNF}, we describe the details of our online multi-object tracking method using the hybrid data association including min-cost multi-commodity flow formulation and its edge costs. Section \ref{Optimization} presents the globally-optimal solution of our model. We report and discuss the experimental results in Section \ref{experiment}, and conclude the paper in Section \ref{conclusion}.

\section{Related Work}
\label{relatdwork}

In multi-object tracking, data association based methods fall into a sub-domain known as the \emph{tracking-by-detection} technique, which has shown impressive tracking performance in unconstrained environments. A thorough review can be found in~\cite{luo2014multiple}.
As evidenced in Section \ref {sec:introduction}, the local association method has aroused considerable research interests. Especially with the success of recurrent neural networks (RNNs) in computer vision community \cite{karpathy2015visualizing}, RNNs-based methods have witnessed significant advances on MOT problems. Based on the pioneer work introduced by Ondruska and Posner \cite{ondruska2016deep}, RNNs-based method quickly sparked significant interest to model the local association, and inspired a number of extensions including \cite{MilanAAAIRNNTracking,alahi2016social,sadeghian2017tracking}. Nevertheless, the RNNs usually comes with high computational and memory demands both during the model training and inference.
We here introduce to explicitly enforce locality into the global data association formulation, and introduce a hybrid   data association framework that is able to integrate the advantages of both local and global association methods.

Maintaining locality for global data association is critical for multi-object tracking performance, since global optimization might scale poorly for the complex scenario and long batches without local constraints. Many global association methods enforce locality by iteratively optimizing trajectories  \cite{huang2008robust,yang2012multi,Milan:2014:CEM,Milan2016PAMI}, or using tracklets (\ie, short-term trajectory fragments) instead of individual detections \cite{zhang2008global,yang2014multi,wang2016tracklet}. However, these strategies are hardly applied to online video streams.
Alternatively, one can divide an online video stream into consecutive batches with temporal sliding windows, and apply global data association to each video batch \cite{berclaz2011multiple,chari2015pairwise,dehghan2015gmmcp}. In order to produce consistent trajectories, the connection between optimized trajectories from adjacent batches need to be considered. However, most existing methods adopt heuristic strategies to connect adjacent batches and can not ensure the optimality of the trajectories.

To retain the ability of handling online data, we turn to explicitly model the target-specific information from previous observations, similar to local data association methods, to cooperate with the global data association over multiple frames. Integrating local and global data association is rarely mentioned in the literature. Lenz \etal~\cite{lenz2015followme} proposed an approximate online solution to the min-cost network flow problem with bounded memory and computation. The local consistency, however, is ignored in the optimization of trajectories. Choi~\cite{choi2015NOMT} proposed a near-online multi-object tracking method to formulate the data association between previously tracked objects and detections in a temporal window. The method has a similar problem setting with ours, while the difference is that a highly non-convex formulation is adopted to select appropriate hypotheses for the objects. The solution heavily relies on both the affinity measures and the generated trajectory hypotheses. In contrast, we use a more compact formulation, \ie, the min-cost multi-commodity flow, to address the hybrid data association. The target-specific information contained in the existing trajectories is incorporated into the flow costs in a natural way, ensuring that the objective is still convex. We also propose an optimization algorithm to the network flow problem, and show its effectiveness in multi-object tracking.


Recently, multi-commodity flow has been introduced into multi-object tracking in \cite{ben2014multi,dehghan2015target}. Ben Shitrit \etal~\cite{ben2014multi}  employed the multi-commodity network to account for different appearance groups which are fixed beforehand. Each appearance group (\eg, a basketball team) is supposed to be a specific commodity in the network, and solving multi-commodity flow problems is able to distinguish different appearance groups during the optimization process. Dehghan \etal~\cite{dehghan2015target} have focused on integrating object detector learning and multi-object tracking, where the multi-commodity network is used to track a fixed number of objects in a short video batch. Our approach is different from these methods in that we use a multi-commodity network to formulate a hybrid data association strategy to handle online data. Furthermore, a high-quality near-optimal solution to the min-cost multi-commodity flow problem can be achieved by an efficient algorithm, especially when the number of objects (commodities) is relatively large. Thus we do not need to heuristically prune the graph \cite{ben2014multi} or iteratively relax the hard constraints \cite{dehghan2015target}.

\begin{figure*}
    \centering
    \includegraphics[width=0.98\textwidth]{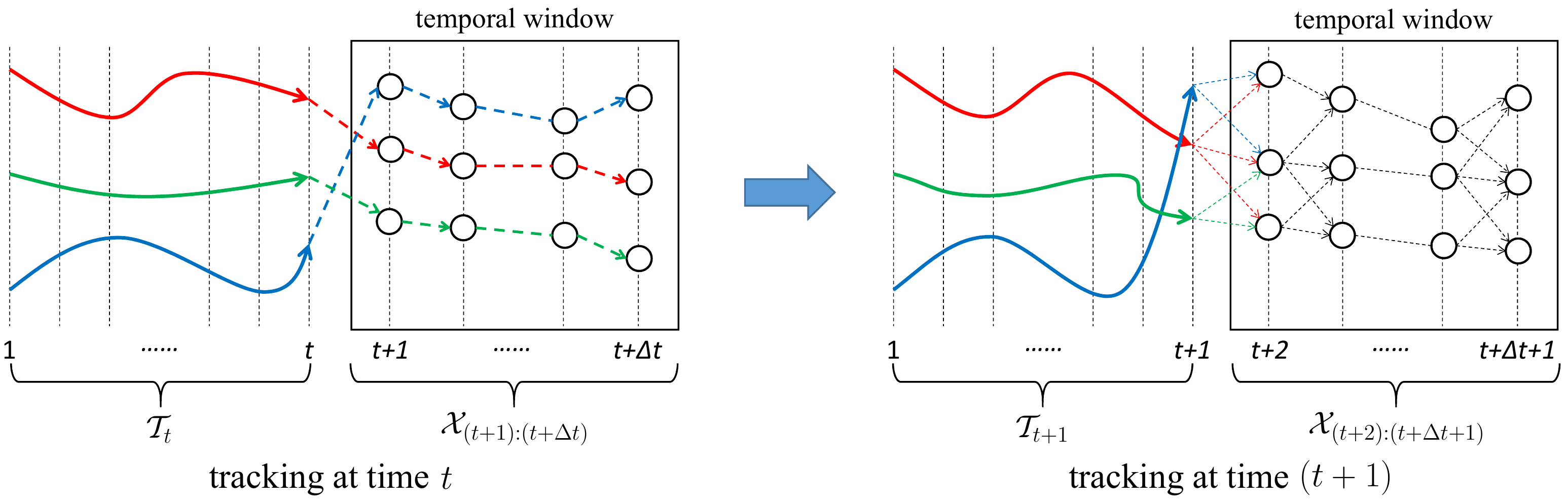}
    \vspace{-0pt}
    \caption{Illustration of the online multi-object tracking process with our hybrid   data association. At each time step $t$, we solve a data association problem between the set of existing trajectories $\mathcal{T}_t$ and the set of detection responses $\mathcal{X}_{(t+1):(t+\Delta t)}$ in a temporal window $[t+1,t+\Delta t]$. After that, the trajectory set $\mathcal{T}_t$ is updated to $\mathcal{T}_{t+1}$ by incorporating the associated detections at frame $(t+1)$, and the temporal window moves \emph{one} time step forward.}
    \vspace{-0pt}
    \label{TrackingProcess}
\end{figure*}


\section{Hybrid Data Association}
\label{MCMCNF}
Let $\mathcal{X} = \{\mathbf{x}_i\}_{i=1}^{N}$ denote the set of detections from the video with $\mathbf{x}_i$ the $i$-th detection and $N$ the number of detections.
Assume that, at each time step $t$, we have a set of existing trajectories $\mathcal{T}_t$ and observe multiple video frames in a temporal window $[t+1,t+\Delta t]$. A set of detection responses $\mathcal{X}_{(t+1):(t+\Delta t)}$ is obtained by applying an object detector to each video frame within the temporal window. The task of \emph{hybrid   data association} is to find globally optimal associations of $\mathcal{T}_t$ over the detections $\mathcal{X}_{(t+1):(t+\Delta t)}$, and simultaneously identify newly appeared objects. Then the trajectory set $\mathcal{T}_t$ is updated to $\mathcal{T}_{t+1}$ by incorporating the associated detections at the frame $(t+1)$, and the temporal window moves \emph{one} time step forward, as shown in Fig.~\ref{TrackingProcess}. In practice, it causes a latency of $(\Delta t - 1)$ to output tracking results, as the trajectories at frame $(t+1)$ is not updated until the frame $(t+\Delta t)$ is observed. Nevertheless, our approach operates in a fully online manner and thus is capable of handling online data. Note that the traditional local or global data association methods can be regarded as special cases of the proposed hybrid framework by adjusting the length of the temporal window as $\Delta t = 1$ or $\Delta t = T$ (total length of the video), respectively.

In this section,  the data association between $\mathcal{T}_t$ and $\mathcal{X}_{(t+1):(t+\Delta t)}$ is formulated as a \emph{min-cost multi-commodity flow} problem, as in Fig.~\ref{MCF-example}. For the convenience of discussion, we drop the time index in the following description, and denote the current set of existing trajectories as $\mathcal{T} = \{T_k\}_{k=1}^{K}$, where $T_k$ is the $k$-th existing trajectory and $K$ is the number of existing trajectories.

\subsection{Our min-cost multi-commodity flow}

Given the set of existing trajectories $\mathcal{T}$ and the set of detections $\mathcal{X} = \{\mathbf{x}_i\}_{i=1}^{N}$, we introduce a directed network $G(\mathcal{X})$ with multiple sources $s_k$ and sinks $n_k$, $k \in \{0,1,\ldots,K\}$.
The directed network $G(\mathcal{X})$ is constructed by the set of detections $\mathcal{X}$ .
Each detection $\mathbf{x}_i \in \mathcal{X}$ corresponds to a pair of nodes $(u_i, v_i)$ in $G$ connected by an \emph{observation edge} with cost $c_i$ and flow $f_i$. The cost $c_i$ indicates the confidence of observing the detection $\mathbf{x}_i$, and the flow $f_i$ encodes the selection of the detection $\mathbf{x}_i$ in some tracks.
Each transition between a pair of detections $(\mathbf{x}_i, \mathbf{x}_j)$ is represented by a \emph{transition edge} $(v_i, u_j)$ with cost $c_{ij}$ and flow $f_{ij}$. The cost $c_{ij}$ represents the coherence between detections $\mathbf{x}_i$ and $\mathbf{x}_j$, and the flow $f_{ij}$ indicates that the two detections are connected through the same track. The set of permissible transitions between detections is denoted as $E$. It could be a subset of all pairs of detections in successive frames by using choice heuristics (\eg, spatial proximity).
Finally, the source $s$ and sink $n$ are introduced with \emph{track start edges} $(s,u_i)$ (with cost $c_{si}$ and flow $f_{si}$) and \emph{track termination edges} $(v_i,n)$ (with cost $c_{it}$ and flow $f_{in}$).
Then the multi-object tracking problem is formulated as sending a set of flows from the source $s$ to sink $n$, which minimizes the total cost
\begin{equation} \label{eq:0-1}
C(f) = \sum_i c_{i}f_{i} + \sum_i c_{si}f_{si} + \sum_{ij \in E} c_{ij}f_{ij} + \sum_i c_{in}f_{in}.
\end{equation}

In this work, each existing trajectory $T_k$ is supposed to be a target-specific commodity $k$ which corresponds to a source-sink pair $(s_k,n_k)$. Specifically, sources $s_k$ and sinks $n_k$ are introduced with track start edges $(s_k,u_i)$ and track termination edges $(v_i,n_k)$ connected to all detections, indicating that the existing trajectories or newly appeared trajectories are allowed to start and terminate at any detection from the temporal window. For each commodity $k$, sending flows from $s_k$ to $n_k$ through the network incurs a specific set of edge costs. Formally, we use $f_{i}^k$, $f_{ij}^k$, $f_{si}^k$, and $f_{in}^k$ to represent the amount of the $k$-th commodity flows on the observation edge $(u_i, v_i)$, the transition edge $(v_i, u_j)$, the track start edges $(s_k,u_i)$, and the track terminate edge $(v_i,n_k)$, respectively. The corresponding edge costs, in a similar way, are denoted as $c_{i}^k$, $c_{ij}^k$, $c_{si}^k$, and $c_{in}^k$.

To identify newly appeared objects, we add a dummy commodity $0$ with the source $s_0$ and sink $n_0$ to represent a target-independent model. We call a flow sent from $s_k$ to $n_k$ the $k$-th commodity flow. That is, the source and sink are extended to account for multiple commodities (see an example in Fig.~\ref{MCF-example}). Then the optimal associations of $T_k$ over $\mathcal{X}$ can be found by sending the $k$-th commodity flow through the network. It leads to a multi-commodity flow problem in the community of network flow \cite{ahuja1993network}.

With the network $G(\mathcal{X})$, the hybrid   data association problem is formulated as finding an optimal set of flows between multiple source and sink pairs $\{(s_k,n_k)\}_{k=0}^{K}$, which \textbf{minimizes} the total cost
\begin{eqnarray} \label{eq:1}
\begin{split}
\sum_{k=0}^K \Big( \sum_i c_{si}^k f_{si}^k + \sum_i c_{i}^k f_{i}^k + \sum_{ij \in E} c_{ij}^k f_{ij}^k + \sum_i c_{in}^k f_{in}^k \Big).
\end{split}
\end{eqnarray}
Intuitively, each flow path connects a set of coherent detections over time and thus can be interpreted as an object track. In practice, the flow should subject to the following constraints to satisfy the physical conditions in a real world:
\begin{align}
&\forall k, \quad f_{i}^k,f_{ij}^k,f_{si}^k,f_{in}^k \in \{0,1\}, \label{eq:2} \\
&\forall k, \quad f_{si}^k + \sum_{j:ji \in E} f_{ji}^k = f_i^k = \sum_{j:ij \in E} f_{ij}^k + f_{in}^k, \label{eq:3} \\
&\forall e\in\{i,ij,si,in\}, \quad \sum_{k=0}^K f_{e}^k \leq 1, \label{eq:4} \\
&\forall k, \quad \sum_{i} f_{si}^k = d_k = \sum_{i} f_{in}^k. \label{eq:5}
\end{align}
The constraint (\ref{eq:2}) is a \emph{edge capacity constraint} which means that each detection belongs to at most one track.
The \emph{flow conservation constraint} (\ref{eq:3}) encodes that the sum of flows arriving at any detection $\mathbf{x}_{i}^{k}$ is equal to the flow of its observation edge $f_{i}^{k}$, which also is the sum of outgoing flows from the detection $\mathbf{x}_{i}^{k}$.
The constraints (\ref{eq:2}), (\ref{eq:3}), and (\ref{eq:4}) ensure that all permissible flows in the network come in the form of flow paths from sources to sinks, and also ensure that there is no overlap between multiple paths. The flow variables $f_{i}^k$, $f_{ij}^k$, $f_{si}^k$, $f_{in}^k$ act as binary indicators taking the value $1$ when the corresponding edge is selected in a flow path of the commodity $k$. The constraint (\ref{eq:5}) restricts the total amount of flows sent from $s_k$ to $n_k$ to be a certain value $d_k$.
Consequently, each flow path in the network can be interpreted as an object track which connects a set of coherent detections over time. A flow path of commodity $k$ with $k \neq 0$ is the success track of the existing trajectory $T_k$ within the temporal window. We thus set $d_k = 1$ for $k \neq 0$ to ensure that each existing trajectory has only one success track. For the dummy commodity, we set $d_0 = 20$ to capture a sufficient number of new objects.

To simplify the notation, we collect the flow variables $f_{i}^k$, $f_{ij}^k$, $f_{si}^k$, $f_{in}^k$ in a long vector $\mathbf{f}^k$ and the edge cost variables $c_{i}^k$, $c_{ij}^k$, $c_{si}^k$, and $c_{in}^k$ in a long vector $\mathbf{c}^k$, respectively. Then the optimization problem that minimizes the cost (\ref{eq:1}) with constraints (\ref{eq:2}), (\ref{eq:3}), (\ref{eq:4}), and (\ref{eq:5}) can be rewritten as
\begin{eqnarray} \label{eq:6}
\begin{split}
&\min_{\mathbf{f}} \  \sum_{k=0}^{K} \left(\mathbf{c}^{k}\right)^{\top} \mathbf{f}^{k} \\
        &s.t. \quad      \forall k, \ \ \mathbf{f}^{k} \geq \mathbf{0}, \\
        &\quad \quad \    \forall k, \ \ U\mathbf{f}^{k} = \mathbf{0}, \\
        &\quad \quad \    \forall k, \ \ V\mathbf{f}^{k} = d_k \mathbf{1}, \\
        &\quad \quad \    \forall k, \ \ \mathbf{f}^{k} \ \textrm{integer}, \\
        &\quad \quad \  \sum_{k=0}^{K} \mathbf{f}^{k} \leq \mathbf{1},
\end{split}
\end{eqnarray}
where the constraints are rearranged into the matrix form. The vectors with all zero and one entries are denoted as $\mathbf{0}$ and $\mathbf{1}$, respectively.


\begin{figure}
    \centering
    \includegraphics[width=0.48\textwidth]{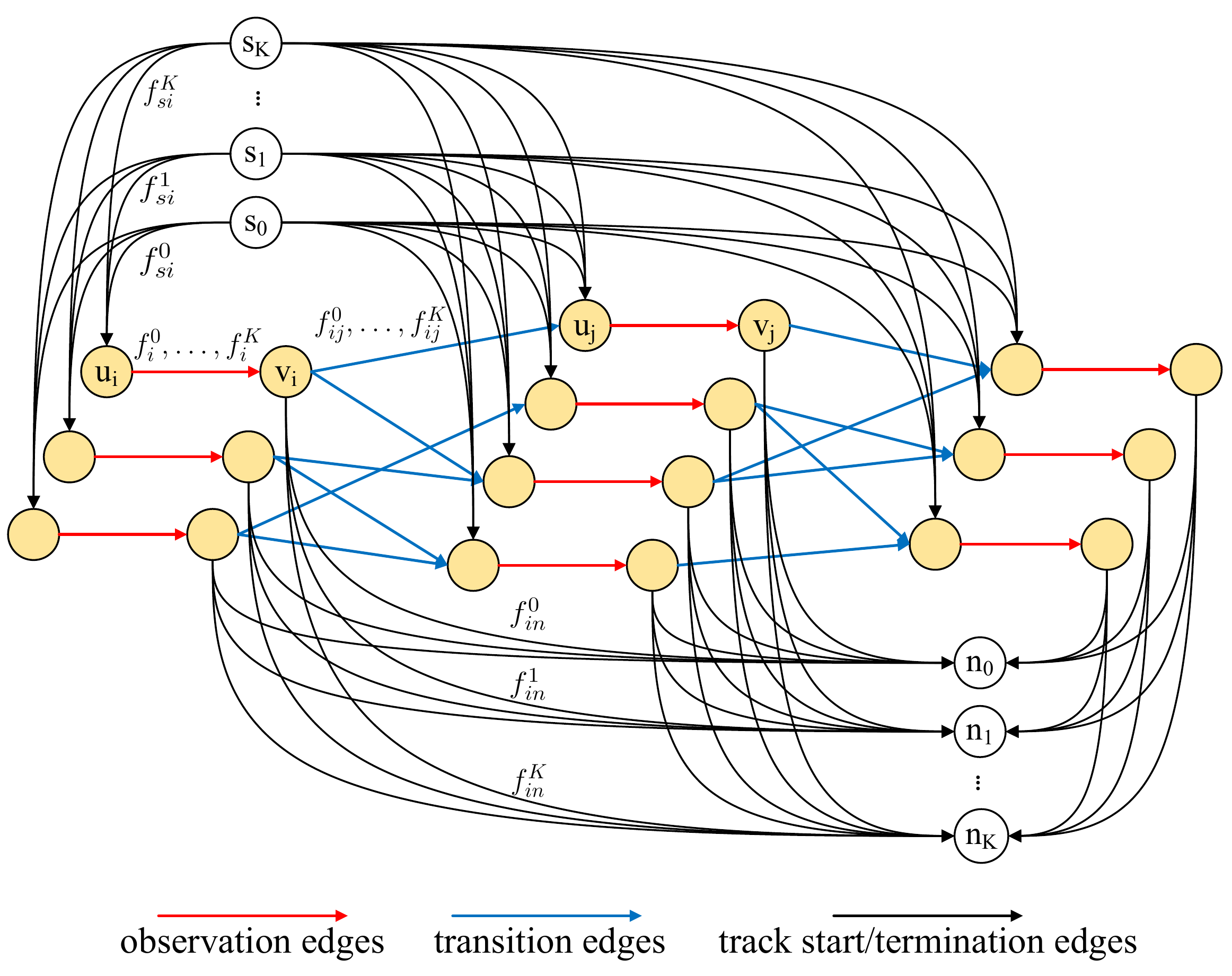}
    \vspace{-0pt}
    \caption{An example of the directed network with multiple sources and sinks. Each detection $\mathbf{x}_i \in \mathcal{X}$ is represented by a pair of nodes connected by an {observation edge}. Possible transitions between detections are modeled by {transition edges}. To allow tracks to start and terminate at any detections from the video, each detection is connected to both a source $s$ and a sink $n$. We use $f_{i}^k$, $f_{ij}^k$, $f_{si}^k$, and $f_{in}^k$ to represent the amount of the $k$-th commodity flows on the observation edge $(u_i, v_i)$, the transition edge $(v_i, u_j)$, the track start edges $(s_k,u_i)$, and the track terminate edge $(v_i,n_k)$, respectively. We add a dummy commodity $0$ with the source $s_0$ and sink $n_0$ to represent a target-independent model}
    \vspace{-0pt}
    \label{MCF-example}
\end{figure}

\subsection{Computing edge costs}

In our min-cost multi-commodity flow formulation, sending flows of a commodity $k$ through the network incurs a specific set of edge costs $\mathbf{c}^k$. Therefore, local information contained in the existing trajectories can be incorporated into the edge costs in a natural way, and thus guides the global data association over multiple video frames. In this subsection, we show that the edge costs can be computed by exploiting the target-specific information from the existing trajectories.



\subsubsection{Observation cost}

Given an existing trajectory $T_k$ and a detection $\mathbf{x}_i$, the observation cost $c_{i}^k$ encodes the possibility of $\mathbf{x}_i$ belonging to $T_k$. $c_{i}^k$ is computed by
\begin{equation} \label{eq:6-1}
c_{i}^k = - \phi_k(\widetilde{\mathbf{a}}_k,\mathbf{a}_i),
\end{equation}
where $\phi_k(\cdot,\cdot)$ is the similarity function used to recognize the specific object corresponding to $T_k$, and $\widetilde{\mathbf{a}}_k$ and $\mathbf{a}_i$ are the appearance feature of the existing trajectory $T_k$ and the detection $\mathbf{x}_i$, respectively. We use Convolutional Neural Network (CNN) features to capture the appearance information of an object, as described in Section~\ref{cnn}. The appearance feature of $T_k$ is represented by the average feature vector over the last $10$ frames, and the appearance feature of $\mathbf{x}_i$ is extracted from the image region corresponding to its location. The similarity function $\phi_k(\cdot,\cdot)$ is involved to assign high similarity scores to pairs of appearance features when both of them originate from the same object corresponding to $T_k$, while producing low similarity scores when more than one of them originate from the other object. We utilize an online similarity learning approach to learn the target-specific similarity function $\phi_k(\cdot,\cdot)$, as described in Section~\ref{OML}. For the dummy commodity, we set $c_{i}^0$ to the negative detector score of the detection $\mathbf{x}_i$.

Note that the observation costs take negative values when the appearance similarity scores or the detector scores are larger than zero, which facilitates the generation of long trajectories. Furthermore, the observation costs taking negative values ensure the appearance consistency for each trajectory since the total cost of the network flows is minimized in our model.

\subsubsection{Transition cost}

The transition cost $c_{ij}^k$ indicates the confidence of connecting the detections $\mathbf{x}_i$ and $\mathbf{x}_j$ in the same success track of $T_k$, which can be computed by
\begin{equation} \label{eq:6-3}
c_{ij}^k = - \phi_k(\mathbf{a}_i,\mathbf{a}_j),
\end{equation}
where $\mathbf{a}_i$ and $\mathbf{a}_j$ are the appearance feature of the detection $\mathbf{x}_i$ and the detection $\mathbf{x}_j$, respectively. For the dummy commodity, the transition cost $c_{ij}^0$ is computed by using the cosine of the angel between two appearance feature vectors as a target-independent similarity function.

\subsubsection{Track start/termination cost}

The track start cost $c_{si}^k$ encodes the possibility that a success track of the $T_k$ starts at the detection $\mathbf{x}_i$. Given the frame index $t_i$ of the detection $\mathbf{x}_i$, we use a constant velocity model to obtain a prediction of $T_k$ at frame $t_i$, denoted as $p(T_k,t_i)$. Then the track start cost $c_{si}^k$ is given by
\begin{equation} \label{eq:66-3}
c_{si}^k = - \eta^{t_i - \psi(T_k)} \cdot o\left(p(T_k,t_i), \mathbf{x}_i\right),
\end{equation}
where $\eta$ is a decay factor (set to 0.95) which discounts long term prediction, $\psi(T_k)$ is the last associated frame of $T_k$, and the function $o$ denotes the overlap rate between two bounding boxes. For the dummy commodity, we set the track start cost $c_{si}^0$ to be a large positive value (10 in our implementation) to reduce the priority of identifying new objects while facilitating the association of the existing trajectories.

Similarly, the track termination cost $c_{in}^k$ encodes the possibility that a success track of the $T_k$ ends at the detection $\mathbf{x}_i$. Assume that an object trajectory ends at all detections with the same probability, we simply set $c_{in}^k = 10$ for all $k$.

\subsection{Online similarity learning}
\label{OML}

Given an existing trajectory $T_k$, we learn a target-specific similarity function $\phi_k(\cdot,\cdot)$ to distinguish the corresponding object from the others. Formally, we use a parametric similarity function that has a bi-linear form to estimate the appearance similarity between two appearance features $\mathbf{x}_i$ and $\mathbf{x}_j$,
\begin{equation} \label{eq:6-4}
\phi_k(\mathbf{a}_i,\mathbf{a}_j) = \mathbf{a}_i^\top \mathbf{W}_k \mathbf{a}_j,
\end{equation}
where $\mathbf{W}_k \in \mathbb{R}^{m \times m}$ with $m$ the dimensionality of appearance features. The task of online similarity learning is to estimate an appropriate parameter matrix $\mathbf{W}_k$ for the existing trajectory $T_k$ in the process of the online tracking.

At each time $t$, we assume that a detection from time $(t+1)$, whose appearance feature is denoted as $\mathbf{a}_k^{(t+1)}$, is associated with the existing trajectory $T_k^{(t)}$. The parameter matrix $\mathbf{W}_k^{(t)}$ of $T_k^{(t)}$ at the current time $t$ is needed to be updated to account for the newly observed appearance feature $\mathbf{a}_k^{(t+1)}$. The principle of updating $\mathbf{W}_k^{(t)}$ is to recognize $\mathbf{a}_k^{(t+1)}$ as a relevant appearance and $\{ \mathbf{a}_l^{(t+1)}|l \neq k \}$ as irrelevant appearances. We therefore construct a set of triplets $\mathcal{S}_k^{(t+1)} = \{(\widetilde{\mathbf{a}}^{(t)}_k,\mathbf{a}^{(t+1)}_k,\mathbf{a}^{(t+1)}_l)| l \neq k \} $, where $\widetilde{\mathbf{a}}_k^{(t)}$ is the appearance feature of $T_k^{(t)}$ at the current time $t$. Each triplet $(a,b,c)$ indicate that the similarity between $a$ and $b$ is apparently larger than the similarity between $a$ and $c$. Forcing the current matrix $\mathbf{W}_k^{(t)}$ to satisfy the triplet set $\mathcal{S}_k^{(t+1)}$ leads to the updated matrix $\mathbf{W}_k^{(t+1)}$ at time $(t+1)$.

We here present an incremental update algorithm to satisfy the triplets sequentially \cite{chechik2010large}. Without loss of generality, assume that we have a parameter matrix $\mathbf{W}^{\tau}$ at the $\tau$-th iteration and observe a triplet $(\mathbf{a}_\tau,\mathbf{a}_\tau^{+},\mathbf{a}_\tau^{-}) $. The goal of incremental updating is to obtain a new matrix $\mathbf{W}$ satisfying
\begin{equation} \label{eq:6-5}
(\mathbf{a}_\tau^{+})^\top \mathbf{W} (\mathbf{a}_\tau^{+}) > (\mathbf{a}_\tau^{-})^\top \mathbf{W} (\mathbf{a}_\tau^{-}) + 1,
\end{equation}
which means that it fulfills the definition of a triplet with a safety margin of $1$. Meanwhile, applying the Passive-Aggressive algorithm \cite{crammer2006online} to maintain smoothness, the new matrix is selected to remain close to the previous matrix $\mathbf{W}^{\tau}$.

We define a hinge loss function to measure the confidence that a matrix $\mathbf{W}$ satisfies the triplet $(\mathbf{a}_\tau,\mathbf{a}_\tau^{+},\mathbf{a}_\tau^{-})$,
\begin{eqnarray} \label{eq:6-6}
\begin{split}
L_{\mathbf{W}}&(\mathbf{a}_\tau,\mathbf{a}_\tau^{+},\mathbf{a}_\tau^{-}) \\
 & = \max\left\{ 0, 1-(\mathbf{a}_\tau^{+})^\top \mathbf{W} (\mathbf{a}_\tau^{+})+(\mathbf{a}_\tau^{-})^\top \mathbf{W} (\mathbf{a}_\tau^{-})  \right\}.
\end{split}
\end{eqnarray}
Then the problem of incremental updating can be expressed as
\begin{eqnarray} \label{eq:6-7}
\begin{split}
&\mathbf{W}^{\tau+1} = \mathop{\arg\min}_{\mathbf{W}} \frac{1}{2} \|\mathbf{W} - \mathbf{W}^{\tau} \|_{F}^2 + C\xi \\
        &s.t. \quad \quad L_{\mathbf{W}}(\mathbf{a}_\tau,\mathbf{a}_\tau^{+},\mathbf{a}_\tau^{-}) \leq \xi, \quad \xi \ge 0,
\end{split}
\end{eqnarray}
where $\|\cdot\|_F$ is the Frobenius norm, $\xi$ is a slack variable, and $C$ is a parameter that controls the trade-off between preserving smoothness and minimizing the loss on the current triplet.

Since Eq.~(\ref{eq:6-7}) is a constrained convex optimization problem, we can directly derive its optimal solution by using the Karush-Kuhn-Tucker (KKT) conditions,
\begin{eqnarray} \label{eq:6-8}
\begin{split}
\left\{
    \begin{aligned}
     &\mathbf{W}^{\tau+1} = \mathbf{W}^{\tau} + \alpha_\tau \mathbf{V}^\tau, \\
     &\mathbf{V}^\tau = \mathbf{a}_\tau (\mathbf{a}_\tau^{+} - \mathbf{a}_\tau^{-})^{\top}, \\
     &\alpha_\tau = \min \left\{C, \frac{L_{\mathbf{W}^{\tau}}(\mathbf{a}_\tau,\mathbf{a}_\tau^{+},\mathbf{a}_\tau^{-}) }{\|\mathbf{V}^\tau\|^2} \right\}.
    \end{aligned}
\right.
\end{split}
\end{eqnarray}
According to Eq.~(\ref{eq:6-8}), the update only happens when the hinge loss $L_{\mathbf{W}^{\tau}}(\mathbf{a}_\tau,\mathbf{a}_\tau^{+},\mathbf{a}_\tau^{-})$ on the triplet is larger than zero.

To summarize, for each existing trajectory $T_k^{(t)}$ at time $t$, we incrementally update the similarity function parameterized by the matrix $\mathbf{W}_k^{(t)}$ through the following steps:
\begin{itemize}\setlength{\itemsep}{-0pt}

\item construct the triplet set $\mathcal{S}_k^{(t+1)}$;

\item sequentially update the matrix  by using the triplet in $\mathcal{S}_k^{(t+1)}$ one-by-one with Eq.~(\ref{eq:6-8});

\item obtain the updated matrix $\mathbf{W}_k^{(t+1)}$ at the time $(t+1)$.
\end{itemize}
Note that the parameter matrix $\mathbf{W}_k$ of the existing trajectory $T_k$ is initialized to an identity matrix when the trajectory is initialization. The incremental update on each iteration, as defined by Eq.~(\ref{eq:6-8}), only involves few matrix operations and thus is extremely efficient. Moreover, the entire online similarity learning process for each trajectory is independent and can be performed parallelly to further improve the computational efficiency.

\section{Optimization}
\label{Optimization}

Finding a global minimum to the hybrid   data association problem (\ref{eq:6}) is exactly an Integer Linear Program (ILP) which is NP-hard. In addition, the optimal solution to its Linear Program (LP) relaxation is not guaranteed to be integral, which serves as an important requirement for the generation of reasonable object trajectories. In this section, by exploring the special structure of the constraints, we propose an efficient optimization algorithm that is able to provide near-optimal integer solutions with empirical sub-optimality certificates.


\subsection{Dantzig-Wolfe decomposition}

Note that most constraints in the problem (\ref{eq:6}) only involve a single commodity, we use the Dantzig-Wolfe decomposition \cite{dantzig1960decomposition} to reformulate the ``relatively easy'' constraints.
Specifically, we consider the nonnegativity constraints $\mathbf{f}^k \geq \mathbf{0}$ and the flow conservation constraints $U\mathbf{f}^k = \mathbf{0}$ that are exactly identical for each commodity $k$.
All feasible flow vectors can be treated as points lying on the polyhedron $P = \{ \mathbf{f} \geq \mathbf{0} \ | \ U\mathbf{f} = \mathbf{0} \}$. It is a \emph{cone} and has a single vertex $\mathbf{0}$ and a finite number of rays $\{ \mathbf{r}^1,\ldots,\mathbf{r}^G \}$. By the Minkowski-Weyl theorem \cite{schrijver1998theory}, we can represent a flow vector $\mathbf{f}^k \in P$ as
\begin{equation} \label{eq:7}
\mathbf{f}^k = \sum\nolimits_{g=1}^{G} \lambda_{k,g} \mathbf{r}^{g},
\end{equation}
where $\lambda_{k,g} \geq 0$ is the associated non-negative coefficient. In our case, the rays $\{ \mathbf{r}^1,\ldots,\mathbf{r}^G \}$ form the basis of the null space defined by the constraint matrix $U$ in the flow conservation constraints $U\mathbf{f} = \mathbf{0}$, which correspond to indicator vectors of all possible paths from the source to the sink in our network.

Substituting the equation (\ref{eq:7}) into (\ref{eq:6}), we can rewrite the formulation as
\begin{eqnarray} \label{eq:8}
\begin{split}
&\min_{\bm{\lambda}} \  \sum_{k=0}^{K} \sum_{g=1}^{G} \lambda_{k,g} \left(\left(\mathbf{c}^{k}\right)^{\top} \mathbf{r}^{g}\right) \\
        &s.t. \quad      \sum_{k=0}^{K} \sum_{g=1}^{G} \lambda_{k,g} \mathbf{r}^{g} \leq \mathbf{1}, \\
        &\quad \quad \ \ \forall k, \ \ \sum_{g=1}^{G} \lambda_{k,g} = d_k, \\
        &\quad \quad \ \ \forall k, \forall g, \ \ \lambda_{k,g} \geq 0, \\
        &\quad \quad \ \ \forall k, \forall g, \ \ \lambda_{k,g} \ \textrm{integer}.
\end{split}
\end{eqnarray}
The formulation (\ref{eq:8}) can be seen as a path flow formulation that is equivalent to the original edge flow formulation (\ref{eq:6}). The variable $\lambda_{k,g}$ is interpreted as the $k$-th commodity flow on the path corresponding to $\mathbf{r}^{g}$, indicating whether the path $\mathbf{r}^{g}$ is selected by the $k$-th commodity or not.


\subsection{Column generation}
\label{CG}

Enumerating all possible paths to construct the complete set $\{ \mathbf{r}^1,\ldots,\mathbf{r}^G \}$ leads to a very large number of
variables for optimization. Actually, only a few paths among $\{ \mathbf{r}^1,\ldots,\mathbf{r}^G \}$ is needed to achieve the optimal solution in practice. We thus use the column generation \cite{ford1958suggested} process to dynamically find the critical paths. In the following, we consider the LP relaxation of (\ref{eq:8}), denoted as the master LP (MLP), by removing the integer constraints, and show later how to obtain a near-optimal integer solution.

Formally, the MLP problem can be expressed as
\begin{eqnarray} \label{eq:8-1}
\begin{split}
\textrm{(MLP)} \quad &\min_{\bm{\lambda}} \  \sum_{k=0}^{K} \sum_{g \in \mathcal{I}} \lambda_{k,g} \left(\left(\mathbf{c}^{k}\right)^{\top} \mathbf{r}^{g}\right) \\
        &s.t. \quad      \sum_{k=0}^{K} \sum_{g \in \mathcal{I}} \lambda_{k,g} \mathbf{r}^{g} \leq \mathbf{1}, \\
        &\quad \quad \ \ \forall k, \ \ \sum_{g \in \mathcal{I}} \lambda_{k,g} = d_k, \\
        &\quad \quad \ \ \forall k, \forall g \in \mathcal{I}, \ \ \lambda_{k,g} \geq 0,
\end{split}
\end{eqnarray}
where $\mathcal{I} = \{1,\ldots,G\}$ is the whole index set of all possible paths. The dual problem of the MLP, denoted as DMLP, has the form
\begin{eqnarray} \label{eq:9}
\begin{split}
\textrm{(DMLP)} \quad &\max_{\bm{\pi},\bm{\sigma}} \  - \mathbf{1}^{\top} \bm{\pi} + \sum_{k=0}^{K} d_k \sigma_k \\
        &s.t. \quad      \forall k, \forall g \in I, \ \ -\bm{\pi}^{\top} \mathbf{r}^{g} + \sigma_k \leq \left(\mathbf{c}^{k}\right)^{\top}\mathbf{r}^{g}, \\
        &\quad \quad \ \ \bm{\pi} \geq \mathbf{0},
\end{split}
\end{eqnarray}
where $(\bm{\pi}, \sigma_k)$ are the dual variables of the primal variables $\lambda_{k,g}$. Due to the duality theory, any dual feasible solution of the DMLP provides a lower bound on the MLP, being the fundamental of the column generation algorithm.

Assume that, at the iteration $\tau$, only a subset of paths $\{\mathbf{r}^g\}_{g \in \mathcal{I}_{\tau}}$ with $\mathcal{I}_{\tau} \subset \mathcal{I}$  available. Solving the MLP on the subset $\mathcal{I}_{\tau}$ gives rise to the restricted master linear program (RMLP),
\begin{eqnarray} \label{eq:8-3}
\begin{split}
\textrm{(RMLP)} \quad &\min_{\bm{\lambda}} \  \sum_{k=0}^{K} \sum_{g \in \mathcal{I}_{\tau}} \lambda_{k,g} \left(\left(\mathbf{c}^{k}\right)^{\top} \mathbf{r}^{g}\right) \\
        &s.t. \quad      \sum_{k=0}^{K} \sum_{g \in \mathcal{I}_{\tau}} \lambda_{k,g} \mathbf{r}^{g} \leq \mathbf{1}, \\
        &\quad \quad \ \ \forall k, \ \ \sum_{g \in \mathcal{I}_{\tau}} \lambda_{k,g} = d_k, \\
        &\quad \quad \ \ \forall k, \forall g \in \mathcal{I}_{\tau}, \ \ \lambda_{k,g} \geq 0.
\end{split}
\end{eqnarray}
Let $\lambda_{k,g}^\ast$ and $(\bm{\pi}^\ast, \sigma_k^\ast)$ be the optimal primal and dual solution to the RMLP, respectively. We need to check whether the optimal solution to the RMLP is also optimal for the MLP, and decide whether the current path set $\mathcal{I}_{\tau}$ is needed to be augmented. It can be realized by solving the following \emph{pricing} problem:
\begin{eqnarray} \label{eq:10}
\begin{split}
\zeta_k = \min \left\{ \left(\mathbf{c}^{k} + \bm{\pi}^\ast \right)^{\top}\mathbf{r}^{g} \ \big| \ g \in \mathcal{I} \right\}.
\end{split}
\end{eqnarray}
In our case, the pricing problem turns into a \emph{shortest path} problem with regard to the modified edge cost $(\mathbf{c}^{k} + \bm{\pi}^\ast)$ for the commodity $k$, which can be solved very efficiently by dynamic programming. With the optimal solution $\zeta_k$ to the pricing problem, we have the following proposition.

\begin{proposition}

If $\zeta_k \geq \sigma_k^\ast$ holds for all $k$, the optimal primal solution to the RMLP $\lambda_{k,g}^\ast$ optimally solves the MLP.

\end{proposition}

\begin{proof}

Given the optimal primal solution to the RMLP $\lambda_{k,g}^\ast$, we can validate that $\lambda_{k,g}^\ast$ is a feasible solution to the MLP by setting $\lambda_{k,g} = 0$ for those paths not in the current set $\mathcal{I}_{\tau}$. Therefore, the optimal value of the RMLP gives an upper bound on the MLP,
\begin{eqnarray} \label{eq:8-4}
\begin{split}
v(RMLP) \geq v(MLP),
\end{split}
\end{eqnarray}
where $v(RMLP)$ and $v(MLP)$ are the optimal value of the RMLP and the MLP, respectively.

Due to the definition of the pricing problem~(\ref{eq:10}), when $\zeta_k \geq \sigma_k^\ast$ holds for all $k \in \{0,1,\ldots,K\}$, we have
\begin{eqnarray} \label{eq:8-5}
\begin{split}
\forall k, \quad \zeta_k = \min \left\{ \left(\mathbf{c}^{k} + \bm{\pi}^\ast \right)^{\top}\mathbf{r}^{g} \ \big| \ g \in \mathcal{I} \right\} \geq \sigma_k^\ast.
\end{split}
\end{eqnarray}
It can be rewritten as
\begin{eqnarray} \label{eq:8-6}
\begin{split}
\forall k, \ \forall g \in \mathcal{I}, \quad -\bm{\pi}^{\ast\top} \mathbf{r}^{g} + \sigma_k^\ast \leq \left(\mathbf{c}^{k}\right)^{\top}\mathbf{r}^{g},
\end{split}
\end{eqnarray}
which implying that the optimal dual solution to the RMLP $(\bm{\pi}^\ast, \sigma_k^\ast)$  is also a feasible solution to the DMLP given by (\ref{eq:9}). Due to the duality theory, the solution $(\bm{\pi}^\ast, \sigma_k^\ast)$ provides a lower (dual) bound on the MLP, we therefore have
\begin{eqnarray} \label{eq:8-7}
\begin{split}
v(RMLP) \leq v(MLP).
\end{split}
\end{eqnarray}
Note that the above equation use the fact that the optimal primal solution $\lambda_{k,g}^\ast$ and the optimal dual solution $(\bm{\pi}^\ast, \sigma_k^\ast)$ to the RMLP give the exactly same optimal value of the objective function.

With the equations (\ref{eq:8-4}) and (\ref{eq:8-7}), we can conclude that the RMLP and the MLP have the same optimal value if $\zeta_k \geq \sigma_k^\ast$ holds for all $k$. Therefore, the optimal primal solution to the RMLP $\lambda_{k,g}^\ast$ optimally solves the MLP. This completes the proof.

\end{proof}

If the condition of the Proposition 1 is not satisfied, \ie, $\zeta_k < \sigma_k^\ast$ for some $k$, the shortest path $\tilde{\mathbf{r}}_k$ provided by the pricing problem (\ref{eq:10}) has a \emph{negative reduced cost}. We introduce $\tilde{\mathbf{r}}_k$ into the subset $\mathcal{I}_{\tau}$, and repeat the process for the next iteration to decrease the objective value of the MLP.


To obtain a near-optimal integer solution to the ILP~(\ref{eq:8}), one can retain the feasible solution with the minimum objective value once the RMLP provides an integer solution during the column generation process (which happens very frequently in practice). Since the optimal solution to the MLP gives a lower bound for the ILP, the difference between the objective value of the returned integer solution and the lower bound is thus an upper bound certificate on its sub-optimality. In our experiments, we obtained small sub-optimality certificates for the returned integer solutions, indicating that our optimization algorithm based on column generation is stable, as summarized in Algorithm~\ref{algorithm:CG}.

\begin{algorithm}[t]
\caption{\ The Hybrid Data Association via Column Generation}
\label{algorithm:CG}
\begin{algorithmic}[1]
\REQUIRE the edge costs $\mathbf{c}^k$ and track numbers $d_k$ for all commodities $k = 0,1,\ldots,K$.
\ENSURE the near-optimal integer solution $\{\mathbf{f}^k\}_{k=0}^K $ to the problem~(\ref{eq:6}) and its sub-optimality certificate $\epsilon$.
\STATE \textbf{Initialize:} \ the initial path set $\mathcal{I}_1$ consists of the shortest paths of all commodities with regard to the edge cost $\mathbf{c}^{k}$.
\FOR{$\tau = 1$ to ITERMAX}
\STATE solve the RMLP defined by (\ref{eq:8-3}) on $\mathcal{I}_{\tau}$ to get the optimal primal and dual solution $\lambda_{k,g}^{\ast} $, $(\bm{\pi}^\ast, \sigma_k^\ast)$;
\STATE $/*$ retain the integer solution $*/$
\IF{$\lambda_{k,g}^{\ast} $ are integer}
\STATE $v(ILP) = v(RMLP)$;
\STATE $\tilde{\lambda}_{k,g} = \lambda_{k,g}^{\ast} $;
\ENDIF
\STATE $/*$ find shortest paths $*/$
\FOR{$k = 0,\ldots,K$}
\STATE $\tilde{\mathbf{r}}_k = \mathop{\arg\min}\nolimits_{\{\mathbf{r}^{g}|g \in \mathcal{I}\}} \left(\mathbf{c}^{k} + \bm{\pi}^\ast \right)^{\top}\mathbf{r}^{g}$;
\STATE $\zeta_k = (\mathbf{c}^{k} + \bm{\pi}^\ast )^{\top}\tilde{\mathbf{r}}_{k}$;
\ENDFOR
\STATE $/*$ optimality check $*/$
\IF{$\zeta_k \geq \sigma_k^\ast$ holds for all $k$}
\STATE break;
\ENDIF
\STATE $/*$ augment the path set $*/$
\STATE $\mathcal{I}_{\tau+1} = \mathcal{I}_{\tau} \cup \left\{\tilde{\mathbf{r}}_k| \forall k, \zeta_k < \sigma_k^\ast\right\}$.
\ENDFOR
\RETURN $\mathbf{f}^k = \sum_{g \in I_{\tau}} \tilde{\lambda}_{k,g} \mathbf{r}^{g} $, $\epsilon = v(ILP) - v(RMLP)$.
\end{algorithmic}
\end{algorithm}

\section{Experiments}
\label{experiment}

In this section, we evaluate our approach on real world videos to demonstrate its effectiveness. Specifically, the performance of our approach is analyzed in three aspects. (i) We evaluate the influence of the length of the temporal window,\ie, $\Delta t$ on multi-object tracking performance for our hybrid   data association framework; (ii) We compare the column generation (CG) solver introduced in this paper and the exact integer linear programming (ILP) solver in terms of sub-optimality, convergence speed, and MOTA score; (iii) We show that our approach produces superior tracking results over the state-of-the-art via both quantitative and qualitative evaluation.

\begin{figure*}[t]
    \centering
    \subfigure[MOTA]{
    \includegraphics[width=0.31\textwidth]{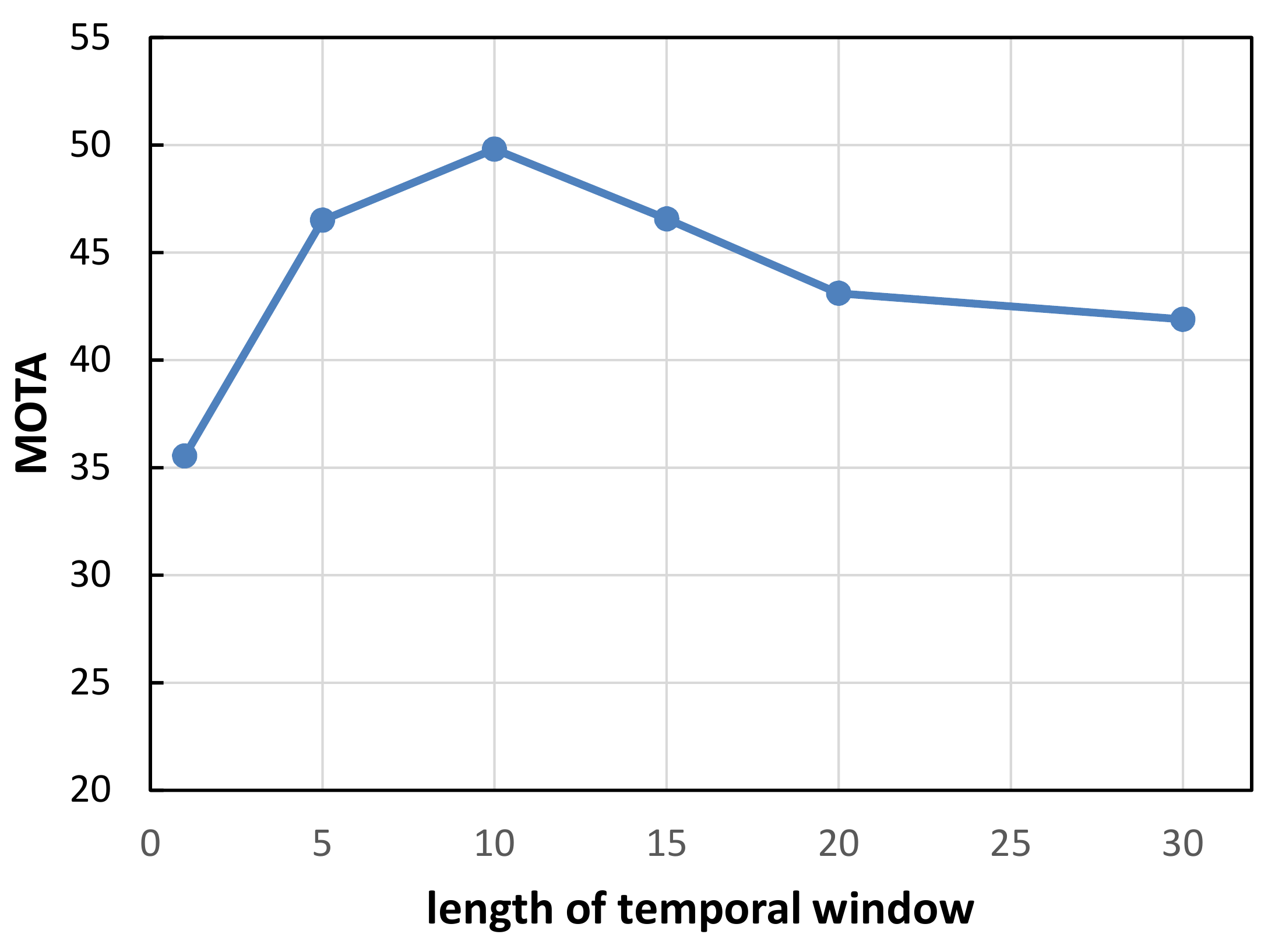}}
    \subfigure[IDS]{
    \includegraphics[width=0.31\textwidth]{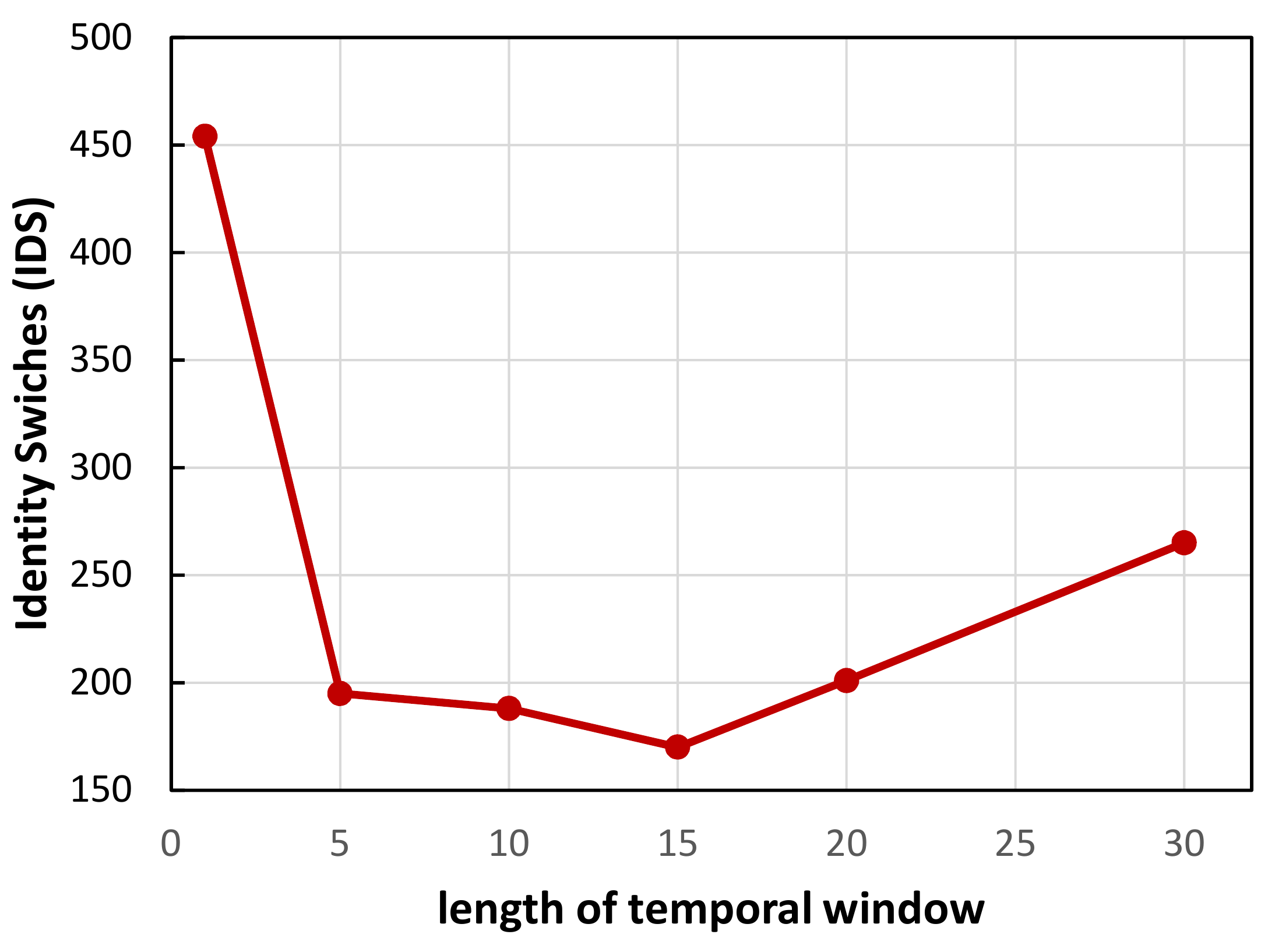}}
    \subfigure[FG]{
    \includegraphics[width=0.31\textwidth]{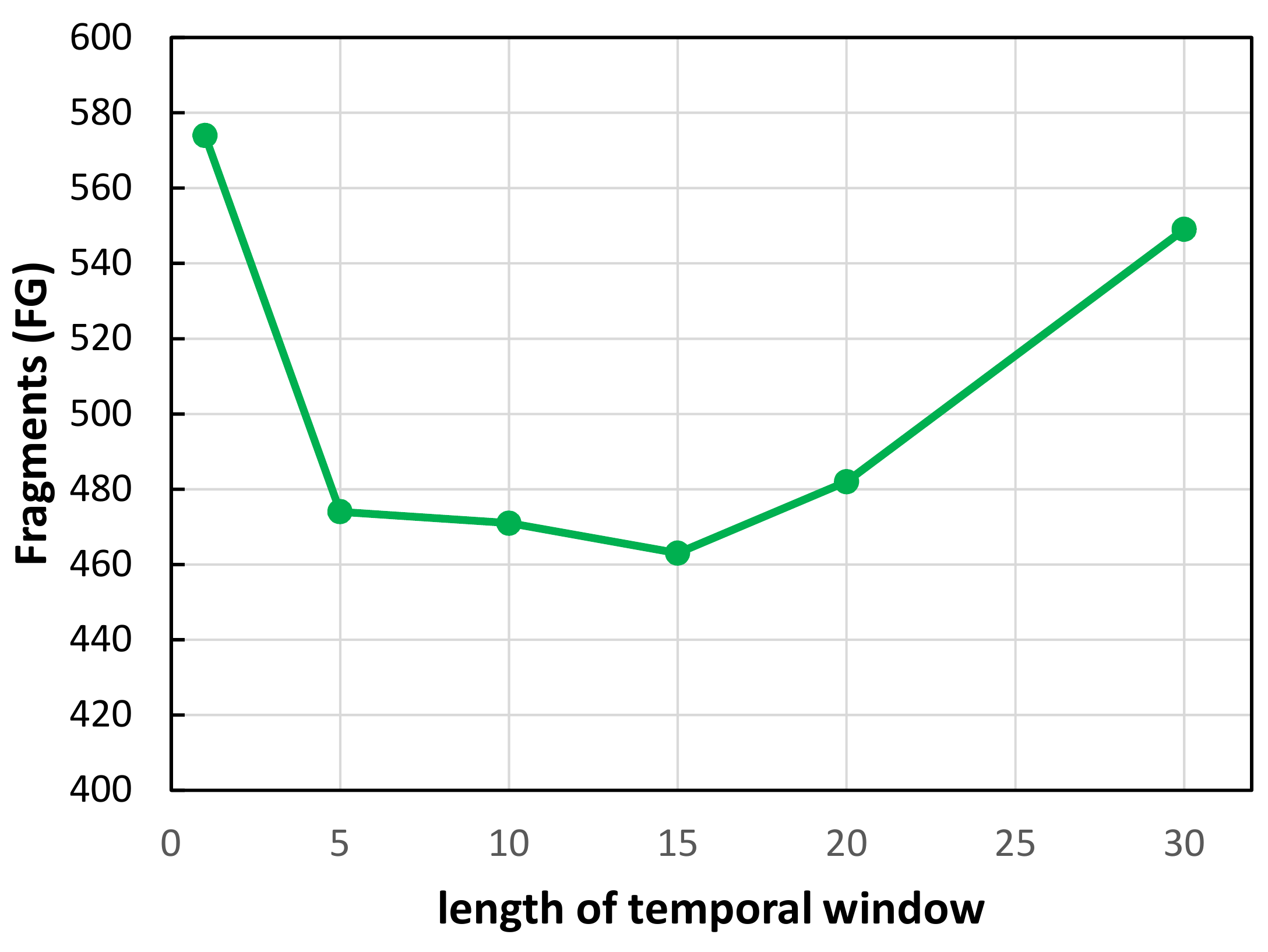}}
    \vspace{-0pt}
    \caption{Influence of the length of the temporal window ($\Delta t$) on tracking performance, the MOTA, IDS, and FG scores on the PETS dataset are plotted. }
    \vspace{-0pt}
    \label{comparison-window}
\end{figure*}

\subsection{Datasets}

We use two publicly available benchmark datasets, \ie, the \emph{PETS 2009} dataset and the \emph{MOTChallenge 2015} dataset, for performance evaluation. The details are listed as follows.

\subsubsection{PETS 2009}

The \emph{PETS 2009} dataset \cite{ellis2009pets} shows an outdoor scene where numerous pedestrians enter, exit, and interact with each other frequently. The images of the dataset are recorded in $768 \times 576$ pixels at $7$ fps. The major challenges of this dataset are frequent occlusions either caused by people interaction or static occlusions due to a traffic sign. Additionally to the widely used \emph{S2L1} and \emph{S2L2} sequence, we also evaluate our approach on the more challenging \emph{S2L3} sequence that captures much denser crowds. The input detections and ground truth of these sequences are from Milan \etal~\cite{Milan:2014:CEM}.

In our experiments, we use the \emph{PETS 2009} dataset for diagnosis analysis, including the investigation of the influence of the critical parameter $\Delta t$ (see Section~\ref{aspect1}) and the comparison between the proposed CG solver and the ILP solver (see Section~\ref{aspect2}). The reason is that the \emph{S2L1}, \emph{S2L2}, and \emph{S2L3} sequences from the \emph{PETS 2009} dataset, respectively, correspond to three representative application scenarios of multi-object tracking with low, high, and crowded object densities.

\subsubsection{MOTChallenge 2015}

The \emph{MOTChallenge 2015} dataset gathers various existing and new challenging video sequences to evaluate the performance of multi-object tracking methods. Since our method performs tracking on the image coordinate, we use the \emph{2D MOT 2015} sequences in the \emph{MOTChallenge 2015}. The sequences are composed of $11$ training and $11$ testing video sequences in which the challenges include camera motion, low viewpoint, varying frame rates, and server weather condition. The training sequences contain over $5500$ frames ($\sim 7$ minutes) and $500$ annotated trajectories ($39905$ bounding boxes). The benchmark releases the ground truth of the training sequences publicly, and thus one can use the training sequences to determine the set of system parameters. The testing sequences contains over $5700$ frames ($\sim 10$ minutes) and $721$ annotated trajectories ($61440$ bounding boxes), while the annotations are not available to avoid (over)fitting of the competing methods to the specific sequences.

Since it is hard for methods to finetune on such a large amount of data, we use the $11$ testing sequences from the \emph{MOTChallenge 2015} dataset for quantitative comparison against various state-of-the-art trackers in our experiments (see Section~\ref{aspect3}). Moreover, the tracking results of all competing methods are automatically evaluated by the benchmark and the performance scores publicly online, making the quantitative comparison strictly fair.

\subsection{Evaluation Metrics}

We use the widely accepted CLEAR MOT performance metrics \cite{keni2008evaluating} for performance evaluation which include the multiple object tracking precision (MOTP$\uparrow$) that measures average overlap rate between estimated trajectories and the ground truth,  the multiple object tracking accuracy (MOTA$\uparrow$) that is a cumulative accuracy combining false positives (FP$\downarrow$), false negatives (FN$\downarrow$) and identity switches (IDS$\downarrow$). We also report performance scores defined by Li \etal~\cite{li2009learning}, including the percentage of mostly tracked (MT$\uparrow$) ground truth trajectories, the percentage of mostly lost (ML$\downarrow$) ground truth trajectories, and the number of times that a ground truth trajectory is interrupted (Frag$\downarrow$). To be specific, a ground truth trajectory is determined to be mostly tracked if and only if it is covered by the estimated trajectories with percentage larger than $80\%$, while a ground truth trajectory is determined to be mostly lost when the coverage percentage is less than $20\%$. Additionally, we report the false positive ratio to account for the accuracy of identifying true targets, which is measured by the number of false alarms per frame (FAF$\downarrow$). Here, $\uparrow$ means that higher scores indicate better results, and $\downarrow$ represents that lower is better.

\subsection{Appearance feature}
\label{cnn}

As for the appearance features, we utilize the region-based CNN features proposed in \cite{girshick2016region}, where the deep neural network is trained on the ImageNet dataset and fine-tuned on the PASCAL VOC dataset. To obtain a more generic deep representation, we follow the strategy in \cite{babenko2015aggregating} to use sum pooling to aggregate the output of the last convolutional layer, rather than directly use the features from the last fully-connected layer. For each detection region, the final feature vector is $256$-dimensional with better time and space complexity. Considering that objects of interest tend to be located close to the geometrical center of an image, we also apply the centering prior to the sum pooling strategy to improve the accuracy, which assigns larger weights to the features from the center of the region.

\subsection{Influence of large temporal window}
\label{aspect1}

The length of the temporal window ($\Delta t$) determines the number of video frames in which the existing trajectories can find their associations, and thus is critical for the proposed hybrid   association framework. Intuitively, taking more frames into account should be helpful for handling inaccurate detections and occlusions. To study the influence of $\Delta t$ on multi-object tracking performance, we conduct an experiment with $\Delta t = \{1,5,10,15,20,30\}$ on the PETS dataset. Fig.~\ref{comparison-window} shows the MOTA, IDS, and FG scores as a function of $\Delta t$.

We can observe from Fig.~\ref{comparison-window} that enlarging the temporal window improvers the overall performance and apparently reduces the number of ID switches and trajectory fragments, especially compared with the purely local method when the length of temporal window is set to $\Delta t = 1$. This result indicates the importance of the data association across multiple frames which our hybrid   data association framework can leverage. As we claimed, integrating the local target-specific model with the global optimization over multiple frames is able to alleviate the irrecoverable errors caused by making decision with only local information. Inaccuracy brought by false alarms and short-term occlusions can be exactly resolved to improve the multi-object tracking performance.

\begin{figure*}[t]
    \centering
    \subfigure[PETS-S2L1]{
    \includegraphics[width=0.31\textwidth]{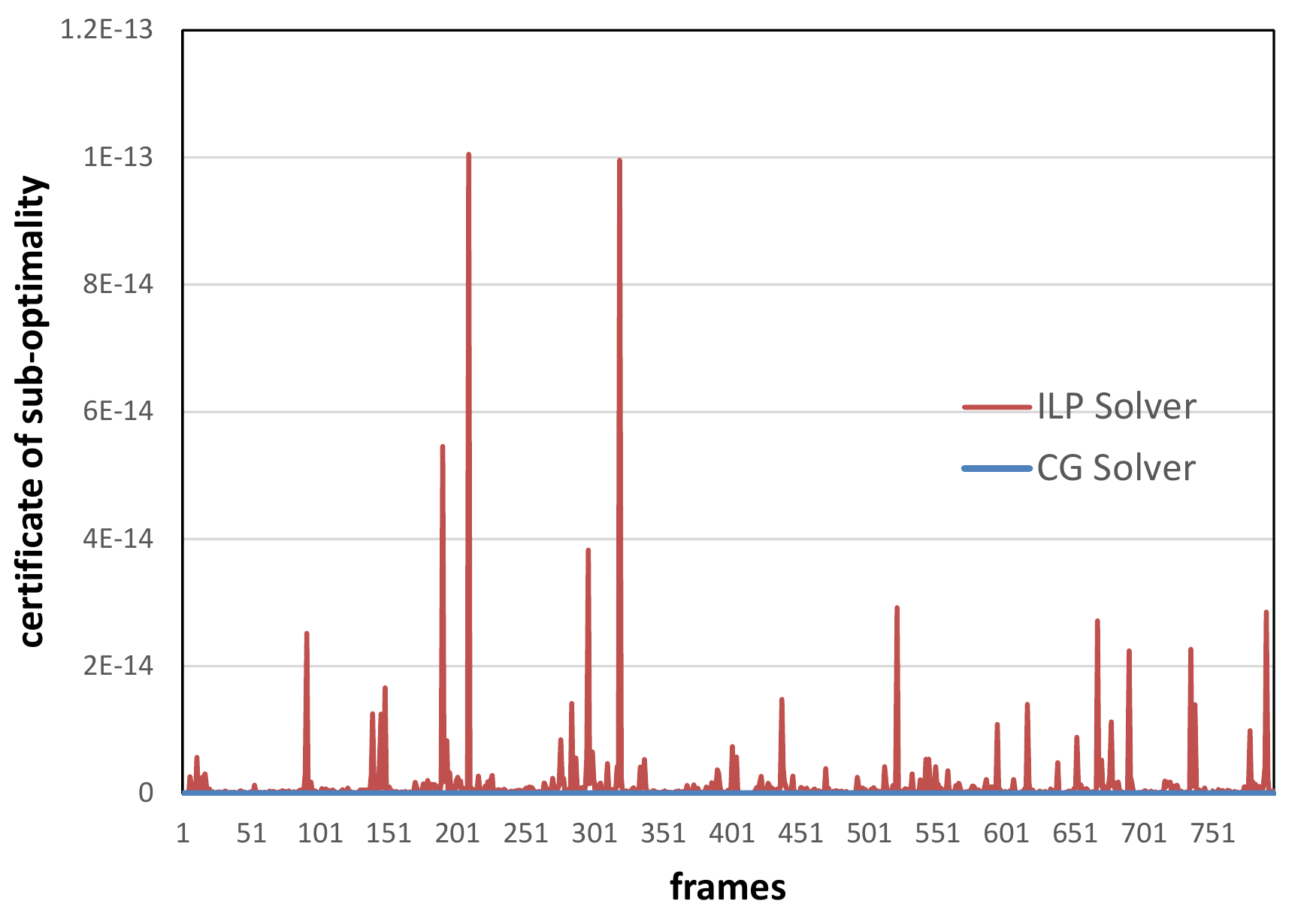}}
    \subfigure[PETS-S2L2]{
    \includegraphics[width=0.31\textwidth]{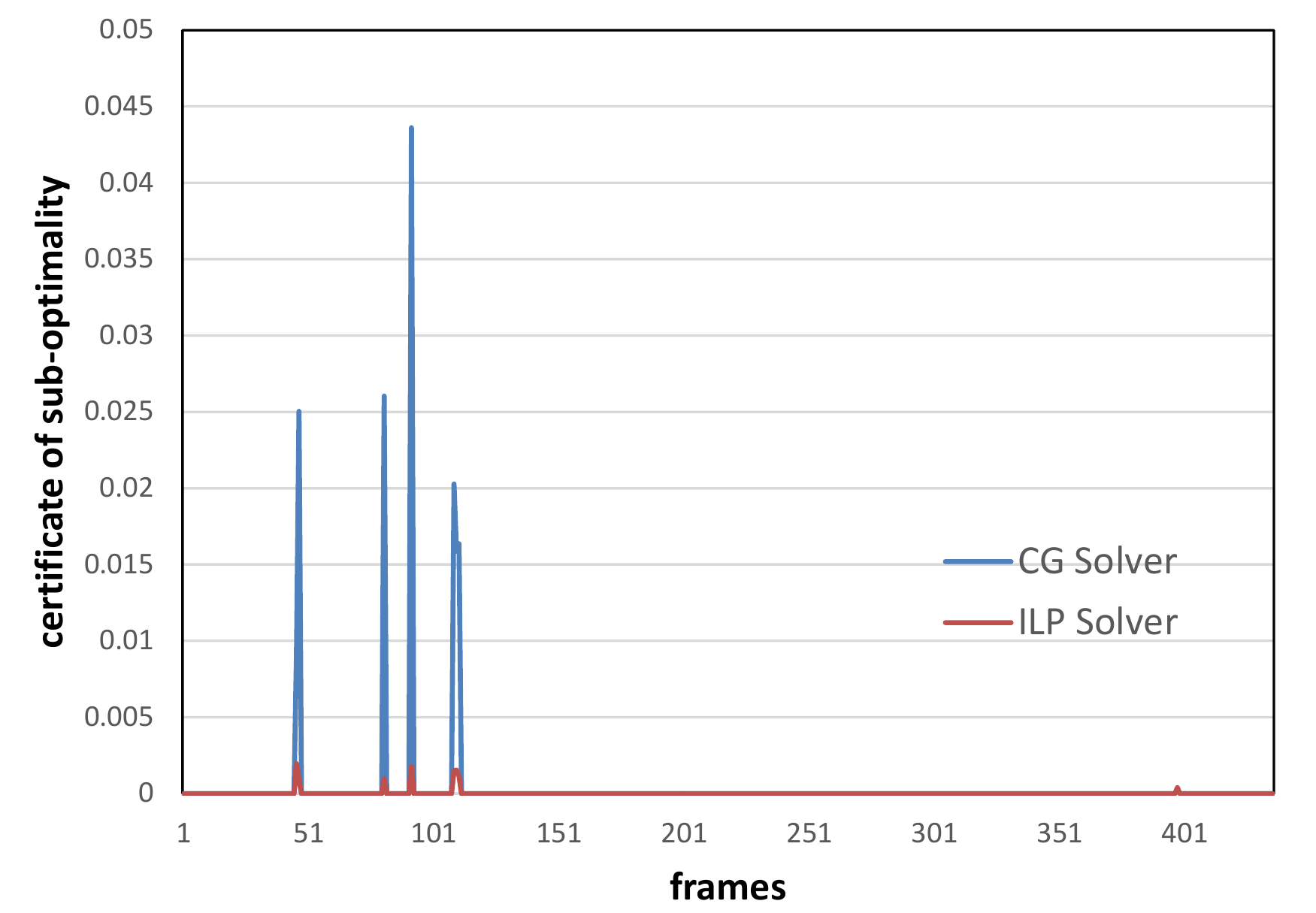}}
    \subfigure[PETS-S2L3]{
    \includegraphics[width=0.31\textwidth]{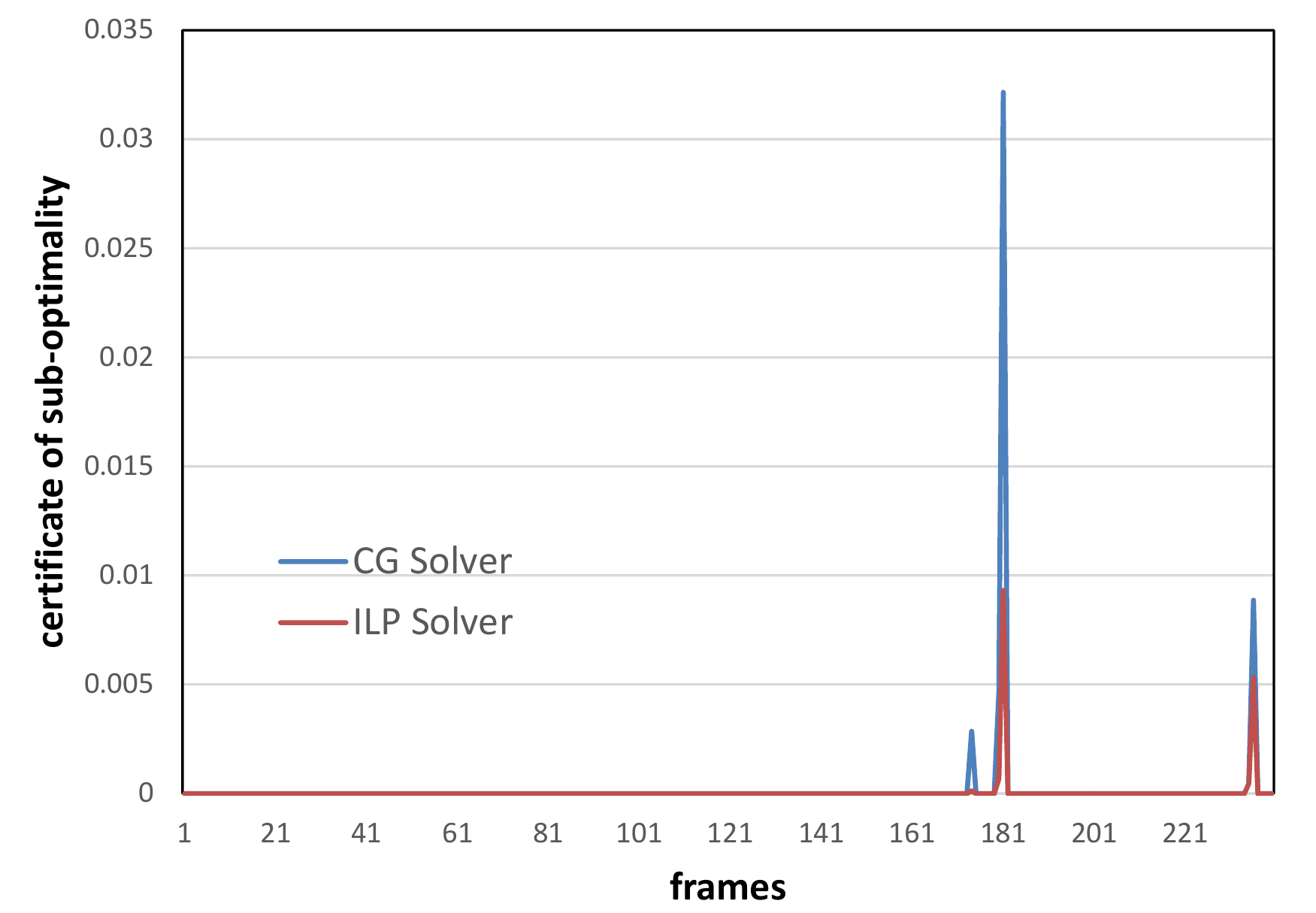}}
    \vspace{-0pt}
    \caption{Sub-optimality comparison between the CG and ILP solvers. The sub-optimality certificates are reported for each frame of the PETS-S2L1, PETS-S2L2, and PETS-S2L3 sequences, respectively. The certificates provided by the CG solver are quite small (equal to zero in most of the cases) and comparable to the ILP solver, indicating that the CG solver is stable}
    \vspace{-0pt}
    \label{comparison-solver}
\end{figure*}

On the other hand, the performance decreases when the temporal window is unduly large ($> 20$). The reason is that the local consistency enforced by target-specific models becomes inaccurate with a long temporal distance. Specifically, due to appearance variations, the target-specific similarity functions obtained by online learning might be inaccurate when they are used to evaluate the object appearances coming from the future. Minimizing the edge costs in the multi-commodity network is therefore unstable to produce consistent flows (trajectories). Similarly, the constant velocity model used to estimate the track start cost might provide unstable long term predictions and thus degrades the tracking accuracy.
To achieve a tradeoff between local consistency and global association, we set $\Delta t = 10$ for our hybrid data association approach and keep it fixed throughout the following experiments.

\subsection{Solver comparison}
\label{aspect2}

\begin{table}[t]
\begin{center}
\caption{Comparison of tracking performance and convergence speed of the CG and ILP solvers on the PETS dataset.}
\vspace{-0pt}
\label{solver}
\small
\begin{tabular}{c|c|c|c|c}
\hline
\multirow{2}{*}{Video}&\multicolumn{2}{c|}{CG Solver}&\multicolumn{2}{c}{ILP Solver} \\
\cline{2-5}
                      &Run time (s)&    MOTA (\%)     &Run time (s)&    MOTA (\%)     \\
\hline
PETS-S2L1     &$0.0276$       &$82.6$       &$0.0296$       &$82.0$        \\
PETS-S2L2     &$0.1571$       &$44.2$       &$0.2131$       &$41.6$        \\
PETS-S2L3     &$0.2381$       &$29.5$       &$0.8156$       &$29.3 $       \\
\hline
\end{tabular}
\vspace{0pt}
\end{center}
\end{table}

In this paper, we introduce a column generation (CG) based solver to the min-cost multi-commodity flow problem in terms of multi-object tracking. Alternatively, one can solve the problem directly using existing integer linear programming packages. To demonstrate the superiority of the proposed CG solver over the standard ILP solver, we report the sub-optimality certificates of the solutions provided by both the CG solver and the ILP solver for the \emph{PETS} dataset in Fig.~\ref{comparison-solver}. The sub-optimality certificates are computed as described in Section~\ref{CG}. For the ILP solver, we employ the commercial software Gurobi which represents the state of the art in ILP.

Overall, the certificates provided by the CG solver are quite small (equal to zero in most of the cases) and comparable to the ILP solver, indicating that the CG solver is stable. As can be observed in Fig.~\ref{comparison-solver}(a), the CG solver provides zero certificates on each frame of the \emph{PETS-S2L1} sequence, while the ILP solver provides certificates much close to zero. It demonstrates that the CG solver exactly finds the optimal integer solution to the min-cost multi-commodity flow problem when the ILP has a tight relaxation to a LP. For the situations where the ILP is not equivalent to a LP, caused by the close interactions of multiple objects, the CG solver provides a near-optimal solution in an efficient way by using a column generation process, as shown in Fig.~\ref{comparison-solver}(b) and Fig.~\ref{comparison-solver}(c).

To further demonstrate the superiority of the CG solver in terms of multi-object tracking, we report the tracking performance (the MOTA score) and convergence speed (the average run time per frame) of both the CG solver and the ILP solver for the three sequences with varying object densities in the \emph{PETS} dataset. Results are shown in Table \ref{solver}. As can be observed, the CG solver achieves better results compared with ILP with significantly faster speed. For each sequence, the CG solver achieves higher MOTA scores than the ILP solver, indicating that the near-optimal solutions produced by the CG solver are much more meaningful for multi-object tracking. It owes to the path-flow reformulation involved in the CG solver which conducts a direct connection between the solution and the estimated trajectories. Furthermore, favorable convergence speed is provided by the CG solver even though the number of objects increases quickly from the sequence PETS-S2L1 ($\thicksim5$ objects per frame) to PETS-S2L3 ($\thicksim30$ objects per frame).

\subsection{Comparison with the state-of-the-art}
\label{aspect3}

We now compare our approach with the state-of-the-art methods on the \emph{MOTChallenge 2015} dataset. The state-of-the-art methods are selected with available corresponding publications at the time of our submission to the test bench, including TC\_ODAL~\cite{BaeY2014robust}, RMOT~\cite{yoon2015bayesian}, MDP~\cite{xiang2015learning}, SCEA~\cite{hong2016online}, TDAM~\cite{yang2016temporal}, DP\_NMS~\cite{pirsiavash2011globally}, SMOT~\cite{dicle2013way}, TBD~\cite{geiger20143d}, CEM~\cite{Milan:2014:CEM}, MotiCon~\cite{leal2014learning}, SegTrack~\cite{milan2015joint}, MHT\_DAM~\cite{kim2015multiple}, JPDA\_m~\cite{hamid2015joint}, TSMLCDE~\cite{wang2016tracklet}, and NOMT~\cite{choi2015NOMT}. Note that the TC\_ODAL, RMOT, MDP, SCEA and TDAM trackers are local data-association methods, the NOMT tracker and our approach perform data association in a hybrid way, while the other trackers are global data-association methods.

\begin{table*}[t]
\begin{center}
\caption{Quantitative comparison results of our approach (denoted as HybridDAT) with other state-of-the-art methods on the \emph{MOTChallenge 2015} dataset. We group the result listings into local, global, and hybrid methods. \txtred{Bold} scores highlight the best results while \txtblu{italic} scores indicate the second best ones. (accessed on 7/6/2016)}
\vspace{-0pt}
\label{quantitative}
\begin{tabular}{c|c|c|c|c|c|c|c|c|c|c}
\hline
 &\textbf{Method} & \textbf{MOTA}[\%]$\uparrow$ & \textbf{MOTP}[\%]$\uparrow$ & \textbf{FAF}$\downarrow$ & \textbf{MT}[\%]$\uparrow$ & \textbf{ML}[\%]$\downarrow$ & \textbf{FP}$\downarrow$ & \textbf{FN}$\downarrow$ & \textbf{IDS}$\downarrow$ & \textbf{FG}$\downarrow$ \\
\hline\hline
&TC\_ODAL~\cite{BaeY2014robust}
&       {$15.1 \pm 15.0$}&       {$70.5$}&       { $2.2$}&       { $3.2$}&       {$55.8$}&       {$12,970$}&       {$38,538$}&       {  $637$}&       {$1,716$}  \\
&RMOT~\cite{yoon2015bayesian}
&       {$18.6 \pm 17.5$}&       {$69.6$}&       { $2.2$}&       { $5.3$}&       {$53.3$}&       {$12,473$}&       {$36,835$}&       {  $684$}&       {$1,282$}  \\
&MDP~\cite{xiang2015learning}
&       {$30.3$ $\pm$ $14.6$}&       {$71.3$}&       { $1.7$}&       {$13.0$}&\txtred{38.4}&       { $9,717$}&       {$32,422$}&       {  $680$}&     {$1,500$}  \\
&TDAM~\cite{yang2016temporal}
&       {$33.0$ $\pm$  $9.8$}&\txtred{72.8}&       { $1.7$}&       {$13.3$}&\txtblu{39.1}&       {$10,065$}&\txtred{30,617}&       {  $464$}&       {$1,506$}  \\
\multirow{-5}{*}{\textit{Local}}
&SCEA~\cite{hong2016online}
&       {$29.1$ $\pm$ $12.2$}&       {$71.1$}&\txtred{ 1.0}&       { $8.9$}&       {$47.3$}&\txtred{ 6,060}&       {$36,912$}&       {  $604$}&       {$1,182$}  \\
\hline
\hline
\multirow{10}{*}{\textit{Global}}
&DP\_NMS~\cite{pirsiavash2011globally}
&       {$14.5$ $\pm$ $13.9$}&       {$70.8$}&       { $2.3$}&       { $6.0$}&       {$40.8$}&       {$13,171$}&       {$34,814$}&       {$4,537$}&       {$3,090$}  \\
&SMOT~\cite{dicle2013way}
&       {$18.2$ $\pm$ $10.3$}&       {$71.2$}&       { $1.5$}&       { $2.8$}&       {$54.8$}&       { $8,780$}&       {$40,310$}&       {$1,148$}&       {$2,132$}  \\
&TBD~\cite{geiger20143d}
&       {$15.9$ $\pm$ $17.6$}&       {$70.9$}&       { $2.6$}&       { $6.4$}&       {$47.9$}&       {$14,943$}&       {$34,777$}&       {$1,939$}&       {$1,963$}  \\
&CEM~\cite{Milan:2014:CEM}
&       {$19.3$ $\pm$ $17.5$}&       {$70.7$}&       { $2.5$}&       { $8.5$}&       {$46.5$}&       {$14,180$}&       {$34,591$}&       {  $813$}&       {$1,023$}  \\
&MotiCon~\cite{leal2014learning}
&       {$23.1$ $\pm$ $16.4$}&       {$70.9$}&       { $1.8$}&       { $4.7$}&       {$52.0$}&       {$10,404$}&       {$35,844$}&       {$1,018$}&       {$1,061$}  \\
&SegTrack~\cite{milan2015joint}
&       {$22.5$ $\pm$ $15.2$}&       {$71.7$}&       { $1.4$}&       { $5.8$}&       {$63.9$}&       { $7,890$}&       {$39,020$}&       {  $697$}&\txtred{  737}  \\
&MHT\_DAM~\cite{kim2015multiple}
&       {$32.4$ $\pm$ $15.6$}&       {$71.8$}&       { $1.6$}&\txtred{16.0}&       {$43.8$}&       { $9,064$}&       {$32,060$}&       {  $435$}&       {  $826$}  \\
&JPDA\_m~\cite{hamid2015joint}
&       {$23.8$ $\pm$ $15.1$}&       {$68.2$}&\txtblu{ 1.1}&       { $5.0$}&       {$58.1$}&\txtblu{ 6,373}&       {$40,084$}&\txtblu{ 365}&       {  $869$}  \\
&TSMLCDE~\cite{wang2016tracklet}
&\txtblu{34.3$\pm$13.1}&       {$71.7$}&       { $1.4$}&\txtblu{14.0}&       {$39.4$}&       { $7,869$}&       {$31,908$}&       { $618$}&       { $ 959$}  \\
\hline
\hline
&NOMT~\cite{choi2015NOMT}
&       {$33.7$ $\pm$ $16.2$}&       {$71.9$}&       { $1.3$}&       {$12.2$}&       {$44.0$}&       { $7,762$}&       {$32,547$}&       {  $442$}&\txtblu{ $ 823$}  \\
\multirow{-2}{*}{\textit{Hybrid}}
&HybridDAT
&\txtred{35.0$\pm$15.0}&\txtblu{72.6}&       { $1.5$}&       {$11.4$}&       {$42.2$}&       { $8,455$}&\txtblu{31,140}&\txtred{  358}&       {$1,267$} \\
\hline
\end{tabular}
\vspace{0pt}
\end{center}
\end{table*}

Table~\ref{quantitative} lists detailed quantitative comparison results on the \emph{MOTChallenge 2015} dataset, where the results are grouped into local, global, and hybrid data-association methods \footnote{The comparison is also available at the website of the MOTChallenge \url{http://motchallenge.net/results/2D_MOT_2015/}.}.
With only the provided detections and a simple dynamic model, our approach shows very competitive performance with the best MOTA score. It demonstrates that our approach performs favorable over the state-of-the-art and is suitable for various unconstrained environments. In particular, the MOTA score and the number of ID switches are substantially improved compared with both local and global data-association methods. It is ascribed to the hybrid   data association framework that is able to find optimal associations for the existing trajectories over multiple video frames. Errors caused by inaccurate detections and occlusions, which are the most challenging issues in complex scenes, are significantly alleviated by our approach to produce consistent trajectories.

As expected, hybrid data-association methods performs better than both local and global methods by a large margin. This superior performance is mainly due to the integration of local target-specific models and global optimization over multiple frames. Compared with the local methods, hybrid data association takes multiple frames into account and therefore is much more stable against noise when association decisions are made. Moreover, compared with the global methods, hybrid data association utilizes the local target-specific models to ensure the local consistency of estimated trajectories, meanwhile retains the ability to handle online data. Benefitting from the superiority of hybrid data association, the NOMT tracker also achieves good scores on the challenging dataset, as we can observed in Table~\ref{quantitative}. In contrast, our approach produces apparently lower FN and IDS scores with a reasonable number of false alarms, and thus provides a better MOTA score. This is because that our min-cost multi-commodity flow formulation  models the multi-object tracking problem in a compact form and enables the efficient near-optimal solution to obtain more accurate trajectories.

On the other hand, our approach produces slightly more fragmented trajectories in return. The reason is that our approach can perform multi-object tracking in an online manner, even though the global optimization over multiple frames are involved. Our approach tends to terminate the trajectory when it has no associated detections in the future frames and thus increases the FG scores. The number of ID switches is significantly reduced due to the consideration of multiple future frames, as shown in Table~\ref{quantitative}.

Several qualitative examples of tracking results produced by our approach on the \emph{MOTChallenge 2015} are shown in Fig.~\ref{sampleresults}. Consistency of the estimated trajectories is indicated by bounding boxes of the same color on the same object over time. Our method is able to accurately track the objects against the inference of abundant false positive detections, short-term occlusions, abrupt motions etc. (Videos suitable for qualitative evaluation of the results across all frames are available at the website of the MOTChallenge \url{http://motchallenge.net/results/2D_MOT_2015/}, as well as the detailed tracking results provided by our approach and the state-of-the-art algorithms.)

\begin{figure*}
    \centering
    \includegraphics[width=0.95\textwidth]{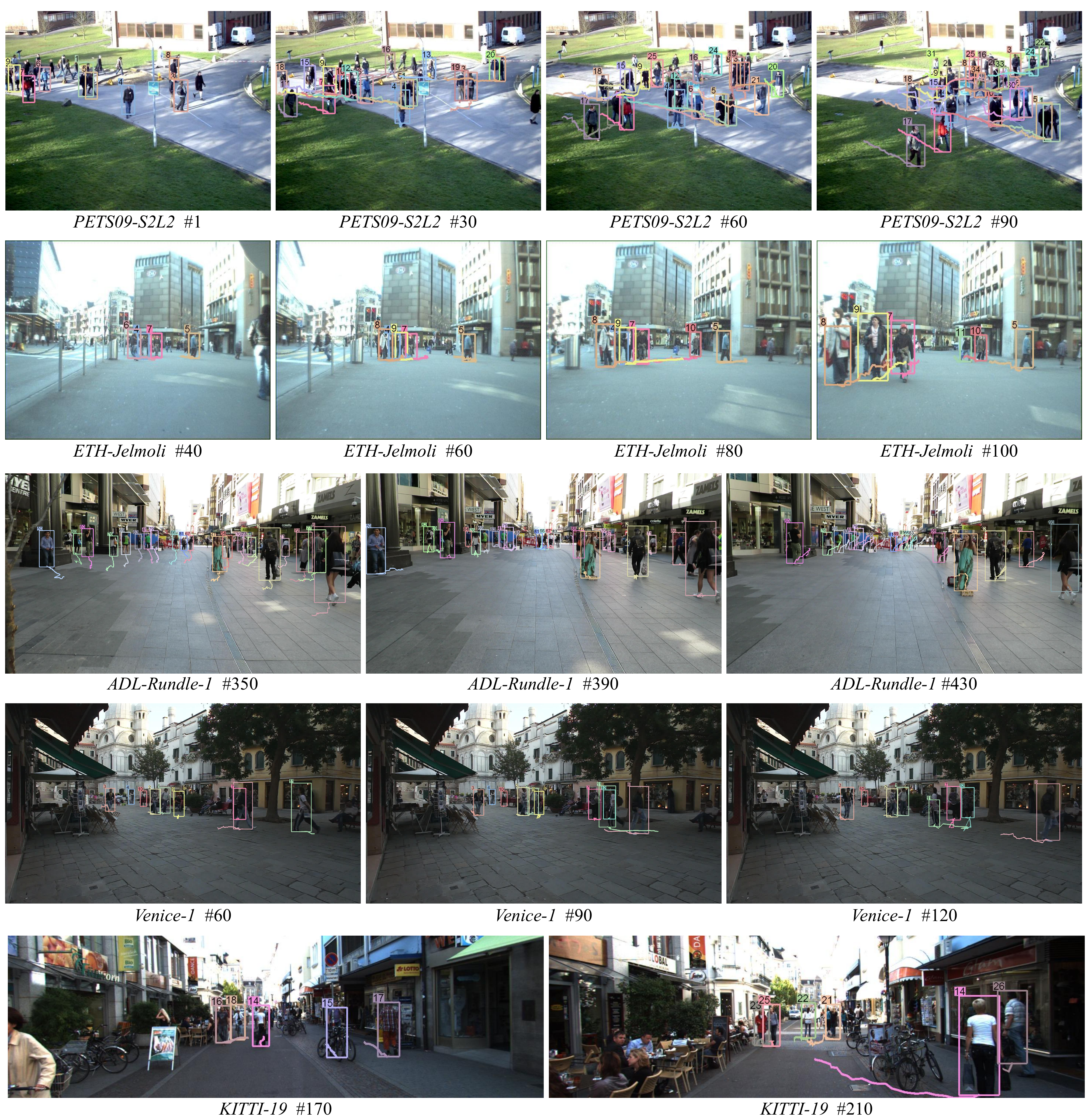}
    \vspace{-0pt}
    \caption{Sample tracking results of our approach on five representative testing video sequences of the MOTChallenge 2015 dataset (\ie, \emph{PETS09-S2L2}, \emph{ETH-Jelmoli}, \emph{ADL-Rundle-1}, \emph{Venice-1}, and \emph{KITTI-19}). At each frame, we show the bounding boxes together with the past trajectories (last 30 frames). The color of the bounding boxes and trajectories indicates the ID of the tracked objects. Best viewed in color. (Refer to the tracking videos for more detailed results.)}
    \label{sampleresults}
\end{figure*}

\section{Conclusion}
\label{conclusion}
In this paper, we have proposed a hybrid data association framework for multi-object tracking. Instead of only considering local associations between adjacent video frames, we explored the superior abilities of global optimization over multiple frames to carry out online tracking. It was formulated as a min-cost multi-commodity flow problem where the local target-specific information is modeled to cooperate with the global association. We employed a powerful online similarity learning algorithm to explicitly build target-specific appearance models to compute the edge costs of our multi-commodity network, improving the discriminative ability of the framework. In addition, we introduced an efficient and effective solution with empirical sub-optimality certificates, and validated its superiority in terms of multi-object tracking. Extensive experiments on various challenging datasets have demonstrated that our approach outperforms the state-of-the-art methods.

Our future work will explore more effective approaches to learn edge costs for the multi-commodity network since it is the most critical issue for good performance. Online similarity learning is just one example of using the appearance cue to compute edge costs, and we believe that our hybrid data association framework can be further improved in terms of multi-object tracking by introducing more useful cues such as motion and shape.

\bibliographystyle{IEEEtran}
\bibliography{IEEEabrv}

\begin{IEEEbiography}[{\includegraphics[width=1in,height=1.25in,clip,keepaspectratio]{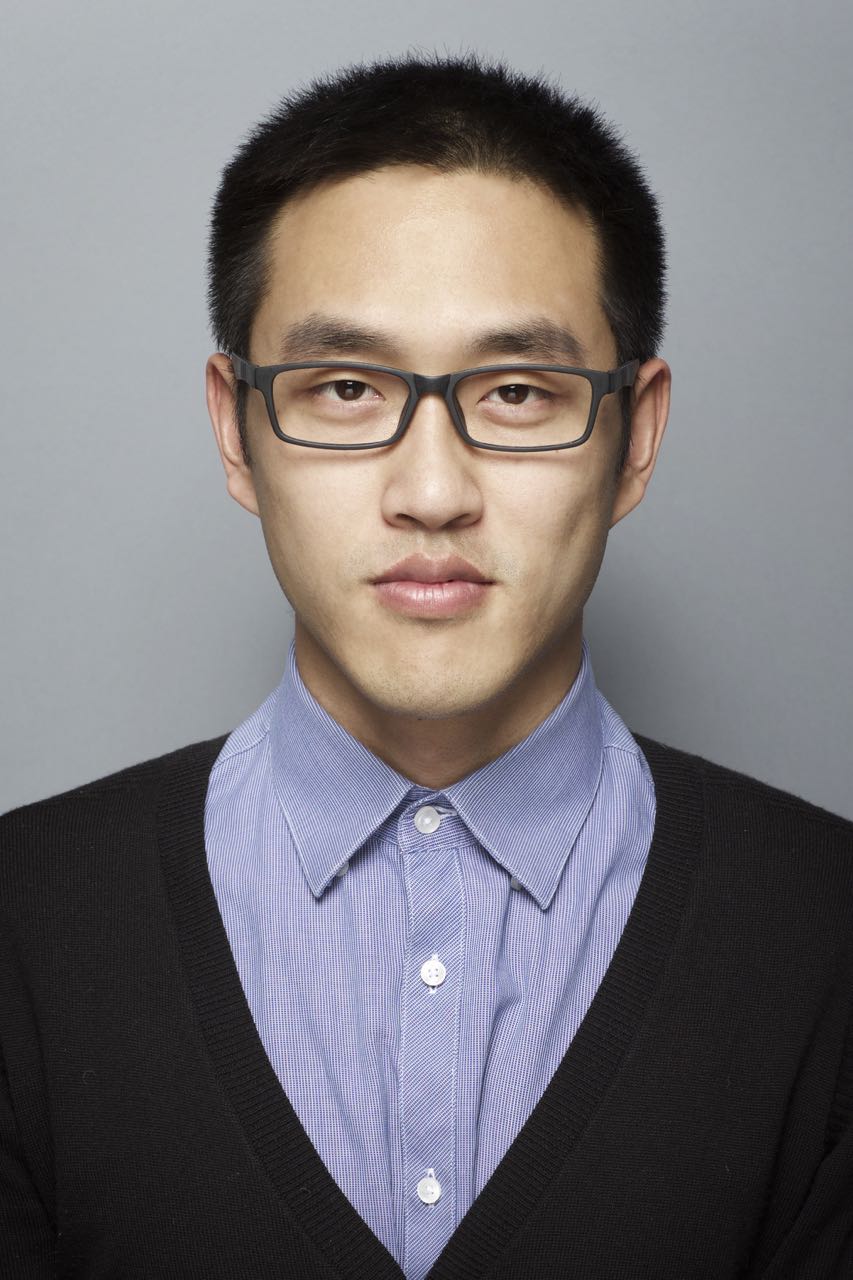}}]{Min Yang}
 received the B.S. degree and Ph.D. degree from Beijing Institute of Technology in 2010 and 2016, respectively.
 His research interests include Computer Vision, Pattern Recognition and Machine Learning.
\end{IEEEbiography}

\begin{IEEEbiography}[{\includegraphics[width=1in,height=1.25in,clip,keepaspectratio]{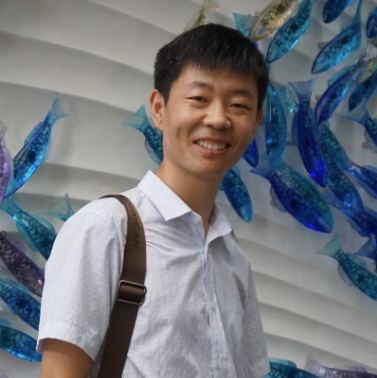}}]{Yuwei Wu} received the Ph.D. degree in computer science from Beijing Institute of Technology (BIT), Beijing, China, in 2014. He is now a research fellow at School of Electrical \& Electronic Engineering, Nanyang Technological University, Singapore. He has strong research interests in computer vision and information retrieval. He received outstanding Ph.D. Thesis award from BIT, and Distinguished Dissertation Award Nominee from China Association for Artificial Intelligence (CAAI). \end{IEEEbiography}

\begin{IEEEbiography}[{\includegraphics[width=1in,height=1.25in,clip,keepaspectratio]{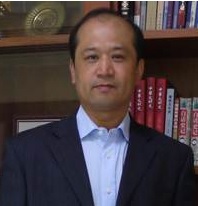}}]{Yunde Jia} (M'11) received the M.S. and Ph.D. degrees in mechatronics from the Beijing Institute of Technology (BIT), Beijing, China, in 1986 and 2000, respectively. He is currently a Professor of computer science with BIT, and serves as the Director of the Beijing Laboratory of Intelligent Information Technology, School of Computer Science. He has previously served as the Executive Dean of the School of Computer Science, BIT, from 2005 to 2008. He was a Visiting Scientist at Carnegie Mellon University, Pittsburgh, PA, USA, from 1995 to 1997, and a Visiting Fellow at the Australian National University, Acton, Australia, in 2011. His current research interests include computer vision, media computing, and intelligent systems. \end{IEEEbiography}

\end{document}